\documentclass[twoside]{article}

\usepackage[accepted]{aistats2016}
\usepackage{hyperref}
\usepackage{url}

\usepackage{algpseudocode}
\usepackage{algorithmicx}
\usepackage{algorithm}
\usepackage{times,amsmath,amssymb,amsfonts,epsfig,graphicx}
\usepackage{amsthm}  
\usepackage{enumerate}
\usepackage{enumitem}
\usepackage{stackrel}
\usepackage{multirow}
\usepackage{graphicx}
\usepackage{subfig}
\usepackage{psfrag}
\usepackage{color}
\usepackage{comment}
\usepackage{bbm}
\usepackage[normalem]{ulem}
\usepackage{xcolor}
\usepackage{framed}
\usepackage{tikz}
\usepackage{booktabs}
\usepackage{wrapfig}
\usepackage{nicefrac}

\newcommand{\savefootnote}[2]{\footnote{\label{#1}#2}}
\newcommand{\repeatfootnote}[1]{\textsuperscript{\ref{#1}}}

\newtheorem{theorem}{Theorem}

\newtheorem{assumption}{Assumption}

\newtheorem{definition}{Definition}

\newtheorem{proposition}{Proposition}
\newtheorem{remark}{Remark}

\numberwithin{equation}{section}

\newcommand{\calC}{\ensuremath{\mathcal{C}}}

\newcommand{\calH}{\ensuremath{\mathcal{H}}}

\newcommand{\calI}{\ensuremath{\mathcal{I}}}
\newcommand{\calP}{\ensuremath{\mathcal{P}}}

\newcommand{\calS}{\ensuremath{\mathcal{S}}}
\newcommand{\calN}{\ensuremath{\mathcal{N}}}

\newcommand{\calV}{\ensuremath{\mathcal{V}}}

\newcommand{\calE}{\ensuremath{\mathcal{E}}}
\newcommand{\calVp}{\ensuremath{\calV^{\prime}}}
\newcommand{\calVpp}{\ensuremath{\calV^{\prime\prime}}}

\newcommand{\norm}[1]{\|{#1}\|}
\newcommand{\abs}[1]{|{#1}|}

\newcommand{\set}[1]{\left\{{#1}\right\}}
\newcommand{\dotprod}[2]{\langle#1,#2\rangle}
\newcommand{\est}[1]{\widehat{#1}}
\newcommand{\expec}{\ensuremath{\mathbb{E}}}
\newcommand{\grad}{\ensuremath{\nabla}}
\newcommand{\hess}{\ensuremath{\nabla^2}}
\newcommand{\matR}{\ensuremath{\mathbb{R}}}
\newcommand{\matZ}{\ensuremath{\mathbb{Z}}}

\newcommand{\argmin}[1]{\underset{#1}{\operatorname{argmin}}}

\newcommand{\prob}{\ensuremath{\mathbb{P}}}
\newcommand{\Linfnorm}{L_{\infty}}
\newcommand{\vp}{\ensuremath{v^{\prime}}}
\newcommand{\vpp}{\ensuremath{v^{\prime\prime}}}
\newcommand{\tp}{\ensuremath{t^{\prime}}}
\newcommand{\lp}{\ensuremath{l^{\prime}}}
\newcommand{\qp}{\ensuremath{q^{\prime}}}

\newcommand{\vecx}{\mathbf x}
\newcommand{\vecn}{\mathbf n}

\newcommand{\vecv}{\mathbf v}
\newcommand{\vecvp}{\mathbf v^{\prime}}
\newcommand{\vecvpp}{\mathbf v^{\prime\prime}}
\newcommand{\vecw}{\mathbf w}

\newcommand{\vecy}{\mathbf y}

\newcommand{\vecz}{\mathbf z}

\newcommand{\matA}{\ensuremath{\mathbf{A}}}

\newcommand{\matV}{\ensuremath{\mathbf{V}}}

\newcommand{\matVpp}{\ensuremath{\mathbf{V^{\prime\prime}}}}

\newcommand{\phitil}{\ensuremath{\tilde{\phi}}}

\newcommand{\numcen}{\ensuremath{m_{x}}}
\newcommand{\numdirec}{\ensuremath{m_{v}}} 
\newcommand{\numdirecp}{\ensuremath{m_{v^{\prime}}}} 
\newcommand{\numdirecpp}{\ensuremath{m_{v^{\prime\prime}}}} 
\newcommand{\numcenpair}{\ensuremath{m^{\prime}_{x}}}
\newcommand{\totsparsity}{\ensuremath{k}} 
\newcommand{\univsparsity}{\ensuremath{k_1}} 
\newcommand{\bivsparsity}{\ensuremath{k_2}} 
\newcommand{\univsupp}{\ensuremath{\calS_1}} 
\newcommand{\bivsupp}{\ensuremath{\calS_2}} 
\newcommand{\bivsuppvar}{\ensuremath{\calS_2^{\text{var}}}} 
\newcommand{\totsupp}{\ensuremath{\calS}}

\newcommand{\thirdtayrem}{\ensuremath{R}}

\newcommand{\ripconst}{\ensuremath{\kappa}}
\newcommand{\spacespl}{\ensuremath{{\calH}}}

\newcommand{\exnoisemag}{\ensuremath{\varepsilon}}
\newcommand{\exnoisemagp}{\ensuremath{\varepsilon^{\prime}}}
\newcommand{\exnoise}{\ensuremath{z}} 
\newcommand{\exnoisevec}{\ensuremath{\mathbf{\exnoise}}}   
\newcommand{\exnoisep}{\ensuremath{z^{\prime}}}  
\newcommand{\dimn}{\ensuremath{d}} 
\newcommand{\degree}{\ensuremath{\rho}} 
\newcommand{\maxdegree}{\ensuremath{\rho_{m}}}
\newcommand{\numdegree}{\ensuremath{\alpha}} 
\newcommand{\smconst}{\ensuremath{B}}
\newcommand{\idenconst}{\ensuremath{D}}
\newcommand{\critintmeas}{\ensuremath{\lambda}}
\newcommand{\gradstep}{\ensuremath{\mu}}
\newcommand{\gradstepp}{\ensuremath{\mu^{\prime}}}
\newcommand{\hessstep}{\ensuremath{\mu_1}}
\newcommand{\taynoisvec}{\ensuremath{\vecn}}
\newcommand{\taynoissca}{\ensuremath{n}}
\newcommand{\baseset}{\ensuremath{\chi}}
\newcommand{\twohashfam}{\ensuremath{\calH_{2}^{d}}}
\newcommand{\thashfam}{\ensuremath{\calH_{t}^{d}}}
\newcommand{\hashfn}{\ensuremath{h}}
\newcommand{\canvec}{\ensuremath{\mathbf{e}}}

\newcommand{\hesssamperr}{\ensuremath{\tau^{\prime}}}
\newcommand{\derivsamperrpp}{\ensuremath{\tau^{\prime\prime}}}

\newcommand{\lpair}{\ensuremath{(l,l^{\prime})}} 
\newcommand{\lpairi}{\ensuremath{(l^{\prime},l)}} 
\newcommand{\qpair}{\ensuremath{(q,q^{\prime})}} 
\newcommand{\qpairi}{\ensuremath{(q^{\prime},q)}}

\newcommand{\xlpair}{\ensuremath{(x_{l},x_{l^{\prime}})}}

\newcommand{\hessestnoisa}{\ensuremath{\mathbf{\eta_{q,1}}}}
\newcommand{\hessestnoisb}{\ensuremath{\mathbf{\eta_{q,2}}}}

\newcommand{\gradtayremf}{\ensuremath{\begin{pmatrix}
  {\vecvp}^{T} \hess \partial_1 f(\zeta_1) \vecvp  \\ \vdots \\ {\vecvp}^{T} \hess \partial_{\dimn} f(\zeta_{\dimn}) \vecvp
 \end{pmatrix}}}

\begin{document}

%

%

\twocolumn[
\aistatstitle{Learning Sparse Additive Models with Interactions in High Dimensions}
\aistatsauthor{ Hemant Tyagi \And Anastasios Kyrillidis \And Bernd G\"{a}rtner \And Andreas Krause }
\aistatsaddress{ ETH Z\"urich \And UT Austin, Texas \And ETH Z\"urich \And ETH Z\"urich } 
]



\begin{abstract}
A function $f: \matR^d \rightarrow \matR$ is referred to as a Sparse Additive Model (SPAM), if it is of the
form $f(\vecx) = \sum_{l \in \totsupp}\phi_{l}(x_l)$, where $\totsupp \subset [\dimn]$, $\abs{\totsupp} \ll \dimn$.
Assuming $\phi_l$'s and $\calS$ to be unknown, the problem of estimating $f$ from its samples has been
studied extensively. In this work, we consider a generalized SPAM, allowing for \emph{second order} interaction terms.
For some $\univsupp \subset [\dimn], \bivsupp \subset {[d] \choose 2}$, the function $f$ is assumed to be of the form: 
$$f(\vecx) = \sum_{p \in \univsupp}\phi_{p} (x_p) + \sum_{\lpair \in \bivsupp}\phi_{\lpair} \xlpair.$$
Assuming $\phi_{p},\phi_{\lpair}$, $\univsupp$ and, $\bivsupp$ to be unknown,
we provide a randomized algorithm that queries $f$ and  
\emph{exactly recovers} $\univsupp,\bivsupp$. Consequently, this also enables us to estimate the underlying 
$\phi_p, \phi_{\lpair}$. We derive sample complexity bounds for our scheme and also extend our analysis
to include the situation where the queries are corrupted with noise -- either stochastic, 
or arbitrary but bounded. Lastly, we provide simulation results on synthetic data, that
validate our theoretical findings.
\end{abstract}
\section{Introduction} \label{sec:intro}
Many scientific problems involve estimating an unknown function $f$, defined over a
compact subset of $\matR^{\dimn}$, with $\dimn$ large. Such problems arise for instance, 
in modeling complex physical processes \cite{Muller2008,Maathuis09,Wainwright09}. Information 
about $f$ is typically available in the form of point values $(x_i,f(x_i))_{i=1}^{n}$, which 
are then used for learning $f$. It is well known that the problem suffers from the curse of dimensionality, if only smoothness 
assumptions are placed on $f$. For example, if $f$ is $C^s$ smooth, then for uniformly approximating 
$f$ within error $\delta \in (0,1)$, one needs $n = \Omega(\delta^{-\dimn/s})$ samples \cite{Traub1988}. 

A popular line of work in recent times considers the setting where $f$ possesses an intrinsic 
low dimensional structure, \textit{i.e.}, depends on only a small subset of $\dimn$ variables. 
There exist algorithms for estimating such $f$ (tailored to the underlying structural assumption),
along with attractive theoretical guarantees that do not suffer from the curse of dimensionality;  
see \cite{Devore2011,Cohen2010,Tyagi2012_nips,Fornasier2010}. One such assumption leads to the 
class of sparse additive models (SPAMs), wherein: 
$$f(x_1, \dots, x_d) = \sum_{l \in \totsupp} \phi_l(x_l),$$ for some unknown 
$\totsupp \subset \set{1,\dots,\dimn}$ with $\abs{\totsupp} = \totsparsity \ll \dimn$. There exist several 
algorithms for learning these models; we refer to \cite{Ravi2009,Meier2009,Huang2010,Raskutti2012,Tyagi14_nips} 
and references therein. 

In this paper, we focus on a generalized SPAM model, where $f$ can also contain a small number of \emph{second order 
interaction terms}, \textit{i.e.},
\begin{equation} \label{eq:intro_gspam_form}
f(x_1,\dots,x_d) = \sum_{p \in \univsupp}\phi_{p} (x_p) + \sum_{\lpair \in \bivsupp}\phi_{\lpair} \xlpair; 
\end{equation}
$\univsupp \subset [\dimn], \bivsupp \subset {[d] \choose 2},$ with $\abs{\univsupp} \ll \dimn,\abs{\bivsupp} \ll \dimn^2$. 
There exist relatively few results for learning models of the form \eqref{eq:intro_gspam_form}, with the existing work being in the regression 
framework \cite{Lin2006, Rad2010, Storlie2011}. Here, $(x_i,f(x_i))_{i=1}^{n}$ are typically samples 
from an unknown probability measure $\prob$. 

We consider the setting where we have the 
freedom to query $f$ at any desired set of points. We propose a strategy for querying $f$, along with an 
efficient recovery algorithm, which leads to much stronger guarantees, compared to those known in the regression setting. 
In particular, we provide the first \emph{finite sample bounds} for exactly recovering sets $\univsupp$ and $\bivsupp$. 
Subsequently, we \emph{uniformly} estimate the individual components: $\phi_p, \phi_{\lpair}$ via additional queries 
of $f$ along the subspaces corresponding to $\univsupp,\bivsupp$. 

\paragraph{Contributions.} We make the following contributions for learning models of the form \eqref{eq:intro_gspam_form}.  
\begin{enumerate}[leftmargin=0.7cm]
\item[$(i)$] Firstly, we provide a randomized algorithm which provably recovers 
$\univsupp, \bivsupp$ \emph{exactly}, with $O(\totsparsity \maxdegree (\log \dimn)^3)$ noiseless point queries. 
Here, $\maxdegree$ denotes the maximum number of occurrences of a variable in $\bivsupp$, and captures the 
underlying \emph{complexity} of the interactions. 
\item[$(ii)$] An important tool in our analysis is a compressive sensing based sampling scheme, 
for recovering each row of a sparse Hessian matrix, for functions that also possess sparse gradients. 
This might be of independent interest.
\item[$(iii)$] We theoretically analyze the impact of additive noise in the point queries on the performance of our algorithm, 
for two noise models: arbitrary bounded noise and independent, identically distributed (i.i.d.) noise. In particular, for additive Gaussian noise,
we show that with $O(\maxdegree^5 \totsparsity^2 (\log \dimn)^4)$ noisy point queries, our algorithm recovers $\univsupp, \bivsupp$ exactly. 
We also provide simulation results on synthetic data that validate our theoretical findings.
\end{enumerate}

\paragraph{Notation.} For any vector $\vecx \in \matR^{\dimn}$, we denote its $\ell_p$-norm by $\norm{\vecx}_p
:= \left ( \sum_{l=1}^\dimn \abs{x_i}^p \right )^{1/p}$. For a set $\calS$, $(\vecx)_{\calS}$ denotes the 
restriction of $\vecx$ onto $\calS$, \textit{i.e.}, $((\vecx)_{\calS})_l = x_l$ if $l \in \calS$ and $0$ otherwise. 
For a function $g: \mathbb{R}^m \rightarrow \mathbb{R}$ of $m$ variables,  
$\expec_p[g]$, $\expec_{\lpair}[g], \expec[g]$ denote expectation with respect to uniform distributions over 
$x_p, (x_l, x_{l^{\prime}})$ and $(x_1,\dots,x_m)$, respectively.
For any compact $\Omega \subset \matR^n$, $\norm{g}_{\Linfnorm (\Omega)}$ denotes the $\Linfnorm$ norm of $g$ in $\Omega$.
The partial derivative operator $\partial/\partial x_i$ is denoted by $\partial_i$. 
For instance, $\partial_1^2 \partial_2 g$ denotes $\partial^3 g/\partial x_1^2 \partial x_2$.
\section{Problem statement}{\label{sec:problem}}
We are interested in the problem of approximating functions $f:\matR^{\dimn} \rightarrow \matR$ 
from point queries. For some unknown sets $\univsupp \subset [\dimn], \bivsupp \subset {[\dimn] \choose 2}$, 
the function $f$ is assumed to have the following form. 
\begin{equation} \label{eq:gspam_form}
 f(x_1,\dots,x_d) = \sum_{p \in \univsupp}\phi_{p} (x_p) + \sum_{\lpair \in \bivsupp}\phi_{\lpair} \xlpair.
\end{equation}
Here, $\phi_{\lpair}$ is considered to be ``truly bivariate'' meaning that $\partial_l \partial_{\lp} \phi_{\lpair} \not\equiv 0$. 
The set of all variables that occur in $\bivsupp$, is denoted by $\bivsuppvar$.
For each $l \in \bivsuppvar$, we refer to $\degree(l)$ as the \emph{degree} of $l$, \textit{i.e.},  
the number of occurrences of $l$ in $\bivsupp$, formally defined as:
\begin{equation*}
\degree(l) := \abs{\set{l^{\prime} \in \bivsuppvar : \lpair \in \bivsupp \ \text{or} \ \lpairi \in \bivsupp}}; \quad l \in \bivsuppvar. 
\end{equation*}
The largest such degree is denoted by $\maxdegree := \max\limits_{l \in \bivsuppvar} \degree(l).$ 

Our goal is to query $f$ at suitably chosen points in its domain, in order to estimate it within the compact 
region\footnote{One could more generally consider the region $[\alpha, \beta]^\dimn$ and transform the variables to 
$[-1,1]^\dimn$ via scaling and transformation.} 
$[-1,1]^{\dimn}$. To this end, note that representation \eqref{eq:gspam_form} is 
not unique\footnote{Firstly, we could add constants to
each $\phi_l, \phi_{\lpair}$, which sum up to zero. Furthermore, for each $l \in \bivsuppvar: \degree(l) > 1$, or 
$l \in \univsupp \cap \bivsuppvar : \degree(l) = 1$, we could add
univariates that sum to zero.}. This is avoided by re-writing \eqref{eq:gspam_form} in the following unique 
ANOVA form \cite{Gu02}: 
\begin{align} \label{eq:unique_mod_rep} 
f(x_1,\dots,x_d) &= c + \sum_{p \in \univsupp}\phi_{p} (x_p) + \sum_{\lpair \in \bivsupp} \phi_{\lpair} \xlpair \nonumber \\ &+ 
\sum_{q \in \bivsuppvar: \degree(q) > 1} \phi_{q} (x_q),
\end{align}
where $\univsupp \cap \bivsuppvar = \emptyset.$ Here, $c = \expec[f]$ and 
$\expec_p[\phi_p] = \expec_{\lpair}[\phi_{\lpair}] = 0$; $\forall p \in \univsupp, \lpair \in \bivsupp$, 
with expectations being over uniform distributions with respect to variable range $[-1,1]$.  
In addition, $\expec_{l}[\phi_{\lpair}] = 0$ if $\degree(l) = 1$. The univariate $\phi_q$ corresponding to 
$q \in \bivsuppvar$ with $\degree(q) > 1$, 
represents the net marginal effect of the variable and has $\expec_q[\phi_q] = 0$. 
We note that $\univsupp, \bivsuppvar$ are disjoint in \eqref{eq:unique_mod_rep}
as each $p \in \univsupp \cap \bivsuppvar$ can be merged with their bivariate counterparts, uniquely.
The uniqueness of \eqref{eq:unique_mod_rep} is shown formally in the appendix.

We assume the setting $\abs{\univsupp} = \univsparsity \ll \dimn$, $\abs{\bivsupp} = \bivsparsity \ll \dimn^2$. 
The set of \emph{all} active variables \textit{i.e.}, $\univsupp \cup \bivsuppvar$ is denoted by $\totsupp$, with 
$\totsparsity: = \abs{\totsupp} = \univsparsity +  \abs{\bivsuppvar}$ being 
the \emph{total sparsity} of the problem. 

Due to the special structure of $f$ in \eqref{eq:unique_mod_rep}, we note that if $\univsupp, \bivsupp$ were 
known beforehand, then one can estimate $f$ via standard results from approximation theory or from 
regression\footnote{This is discussed later.}. 
Hence, our primary focus in the paper is to recover $\univsupp, \bivsupp$.
Our main assumptions for this problem are listed below.

\begin{assumption}
$f$ can be queried from the slight enlargement: $[-(1+r),(1+r)]^d$, for some small $r > 0$.
\end{assumption}
%
%
\begin{assumption}\label{assum:smooth} 
Each $\phi_{\lpair},\phi_p$ is three times continuously differentiable, 
within $[-(1+r),(1+r)]^2$ and $[-(1+r),(1+r)]$ respectively. Since these domains are compact, 
there exist constants $\smconst_m\geq 0$ ($m=0,1,2,3$) so that:
\begin{align*}
\norm{\partial_l^{m_1} \partial_{l^{\prime}}^{m_2} \phi_{\lpair}}_{\Linfnorm[-(1+r),(1+r)]^2} \leq \smconst_m; \ m_1 + m_2 = m,
\end{align*} where $\lpair \in \bivsupp,$ and
\begin{align*}
\norm{\partial_p^{m} \phi_{p}}_{\Linfnorm[-(1+r),(1+r)]} \leq \smconst_m,
\end{align*} where $p \in \univsupp \ \text{or}, \ p \in \bivsuppvar \text{ and } \ \degree(p) > 1.$
\end{assumption}

Our next assumption is for identifying $\univsupp$. 

\begin{assumption} \label{assum:actvar_iden} 
For some constants $\idenconst_1, \critintmeas_1 > 0$, 
we assume that for each $p \in \univsupp$, $\exists$ connected $\calI_p \subset [-1,1]$,
of Lebesgue measure at least $\critintmeas_1 > 0$, such that $\abs{\partial_p \phi_p(x_p)} > \idenconst_1$, $\forall x_p \in \calI_p$.
This assumption is in a sense necessary. If say $\partial_p \phi_{p}$ was zero throughout $[-1,1]$, then it implies that $\phi_{p} \equiv 0$, 
since each $\phi_p$ has zero mean in \eqref{eq:unique_mod_rep}.
\end{assumption}

Our last assumption concerns the identification of  $\bivsupp$. 

\begin{assumption} \label{assum:pair_iden} 
For some constants $\idenconst_2, \critintmeas_2 > 0$, we assume that 
for each $\lpair \in \bivsupp$, $\exists$ connected $\calI_{l}, \calI_{l^{\prime}} \subset [-1,1]$, each interval of 
Lebesgue measure at least $\critintmeas_2 > 0$, such that 
$\abs{\partial_l \partial_{l^{\prime}} \phi_{\lpair} \xlpair}  > \idenconst_2, \ \forall \xlpair \in \calI_{l} \times \calI_{l^{\prime}}$.
\end{assumption}

Given the above, our problem specific parameters are: $(i)$ $\smconst_i$; $i=0,..,3$, 
$(ii)$ $\idenconst_j,\critintmeas_j$; $j=1,2$ and, $(iii)$ $\totsparsity, \maxdegree$.
We do not assume $\univsparsity, \bivsparsity$ to be known, but instead assume that $\totsparsity$ is known. 
Furthermore it suffices to use estimates for the problem 
parameters instead of exact values: In particular, we can use upper bounds for: $\totsparsity, \maxdegree$, $\smconst_i$; $i=0,..,3$ and lower bounds for: 
$\idenconst_j,\critintmeas_j$; $j=1,2$. 
\section{Our sampling scheme and algorithm} \label{sec:algo}
We start by explaining our sampling scheme, followed by our algorithm for identifying $\univsupp,\bivsupp$. 
Our algorithm proceeds in two phases -- we first estimate $\bivsupp$ and then $\univsupp$. Its theoretical properties for the 
\emph{noiseless} query setting are described in Section \ref{sec:noiseless_query_res}. Section \ref{sec:noise_impact} then analyzes how the sampling 
conditions can be adapted to handle the \emph{noisy} query setting. 
\subsection{Sampling scheme for estimating $\bivsupp$} 
Our main idea for estimating $\bivsupp$ is to estimate the off-diagonal entries of the Hessian of $f$, at appropriately chosen points. 
The motivation is the observation that for any 
$\lpair \in {[\dimn] \choose 2}$: 
$$\partial_l \partial_{\lp} f =  \left\{
	\begin{array}{ll}
		\partial_l \partial_{\lp} \phi_{\lpair}  & \mbox{if } \lpair \in \bivsupp, \\
		0 & \mbox{otherwise.}
	\end{array}
\right.$$ To this end, consider the Taylor expansion of 
the gradient $\grad f$, at $\vecx \in \matR^{\dimn}$, along the direction $\vecvp \in \matR^{\dimn}$, 
with step size $\hessstep$. Since $f$ is $C^3$ smooth, we have for $\zeta_q = \vecx + \theta_q \vecvp$, 
for some $\theta_q \in (0,\hessstep)$, $q = 1,\dots,\dimn$:
\begin{align} \label{eq:grad_tay_exp_f}
&\frac{\grad f(\vecx + \hessstep\vecvp) - \grad f(\vecx)}{\hessstep} \nonumber \\ &= \hess f(\vecx) \vecvp + \frac{\hessstep}{2} \gradtayremf.
\end{align} 
We see from \eqref{eq:grad_tay_exp_f} that the $l^{th}$ entry of $(\grad f(\vecx + \hessstep\vecvp) - \grad f(\vecx))/\hessstep$,
corresponds to a ``noisy'' linear measurement of the $l^{th}$ row of $\hess f(\vecx)$ with $\vecvp$. The noise corresponds 
to the third order Taylor remainder terms of $f$. 

Denoting the $l^{\text{th}}$ row of $\hess f(\vecx)$ by 
$\grad \partial_l f(\vecx) \in \matR^{\dimn}$, we make the following crucial observation: 
if $l \in \bivsuppvar$ then $\grad \partial_l f(\vecx)$ has at most $\maxdegree$ non-zero \emph{off-diagonal} entries,  
implying that it is $(\maxdegree+1)$ sparse. This follows on account of the structure of $f$ \eqref{eq:unique_mod_rep}.
Furthermore, if $l \in \univsupp$ then $\grad \partial_l f(\vecx)$ has at most 
one non zero entry (namely the diagonal entry), while if $l \notin \totsupp$, then $\grad \partial_l f(\vecx) \equiv 0$.

\paragraph{Compressive sensing based estimation.} Assuming for now that we have access to an oracle that provides 
us with gradient estimates of $f$, 
this suggests the following idea. We can obtain random linear measurements, for \emph{each row} 
of $\hess f(\vecx)$ via gradient differences, as in \eqref{eq:grad_tay_exp_f}. As each row is 
sparse, it is known from compressive sensing (CS) \cite{Candes2006,Donoho2006} 
that it can be recovered with only a few 
measurements. 

Inspired by this observation, consider an oracle that provides us with the estimates: 
$\est{\grad}{f}(\vecx), \{\est{\grad}{f}(\vecx+\hessstep\vecvp_j)\}_{j=1}^{\numdirecp}$ where 
$\vecvp_j$ belong to the set:
\begin{align*} 
\calVp := \{\vecvp_j \in \matR^{\dimn} : \vp_{j,q} &= \pm1/\sqrt{\numdirecp} \ \text{w.p.} \ 1/2 \ \text{each};
 \nonumber \\  j&=1,\dots,\numdirecp \ \text{and} \ q=1,\dots,{\dimn}\}. 
\end{align*}
Let $\est{\grad} f(\vecx) = \grad f(\vecx) + \vecw(\vecx)$, where $\vecw(\vecx) \in \matR^{\dimn}$ denotes the
gradient estimation noise. Denoting $\matV^{\prime} = [\vecvp_1 \dots \vecvp_{\numdirecp}]^T$, 
we obtain $\dimn$ linear systems, by employing \eqref{eq:grad_tay_exp_f} at each $\vecvp_j \in \calVp$:  
\begin{align} \label{eq:hessrow_est_linfin}
\vecy_q = \matV^{\prime} \grad \partial_q f(\vecx) + \hessestnoisa + \hessestnoisb; \quad q=1,\dots,\dimn.
\end{align} 
$\vecy_q \in \matR^{\numdirecp}$ represents the measurement vector for the $q^{\text{th}}$ row, 
with $$(\vecy_q)_j = ((\est{\grad} f(\vecx + \hessstep\vecvp_j) - \est{\grad} f(\vecx))_q)/\hessstep$$ while 
$\hessestnoisa, \hessestnoisb \in \matR^{\numdirecp}$ represent noise with 
$(\hessestnoisa)_j = (\hessstep/2) {\vecvp_j}^{T} \hess \partial_q f(\zeta_{q,j}) \vecvp_j$ and 
$(\hessestnoisb)_j = (w_q(\vecx + \hessstep\vecvp_j) - w_q(\vecx))/\hessstep.$
Given the measurement vector $\vecy_q$, we can then obtain the estimate $\est{\grad} \partial_q f(\vecx)$ 
individually for each $q = 1,\dots,\dimn$, via $\ell_1$ minimization \cite{Candes2006,Donoho2006,Wojta2012}. 

\paragraph{Estimating sufficiently many Hessian's.} Having estimated \emph{each row} of $\hess f$ at some fixed $\vecx$, 
we have at hand an estimate of the set: $\{\partial_i \partial_{j} f(\vecx) : (i,j) \in {[\dimn] \choose 2} \}$. 
Our next goal is to repeat the process, at sufficiently many $\vecx$'s within $[-1,1]^{\dimn}$. 

We will denote the set of such points as $\baseset$. 
This will then enable us to sample each underlying $\partial_l \partial_{\lp} \phi_{\lpair}$ within 
its respective critical interval, as defined in Assumption \ref{assum:pair_iden}. Roughly speaking, 
since $\abs{\partial_l \partial_{\lp} \phi_{\lpair}}$ is ``suitably large'' in such an interval, 
we will consequently be able to detect each $\lpair \in \bivsupp$, via a thresholding procedure. To this end, 
we make use of a family of hash functions, defined as follows.
\begin{definition} \label{def:thash_fam}
For some $t \in \mathbb{N}$ and $j=1,2,\dots$, let $h_j : [\dimn] \rightarrow \set{1,2,\dots,t}$.
Then, the set $\thashfam = \set{\hashfn_1,\hashfn_2,\dots}$ is a $(\dimn,t)$-hash family if for 
any distinct $i_1,\dots,i_t \in [\dimn]$, $\exists$ $\hashfn \in \thashfam$ such that $h$ is an 
injection when restricted to $i_1,i_2,\dots,i_t$.
\end{definition}
%
%
\begin{figure}
     \centering
     \subfloat[][]{\includegraphics[width=0.6\linewidth]{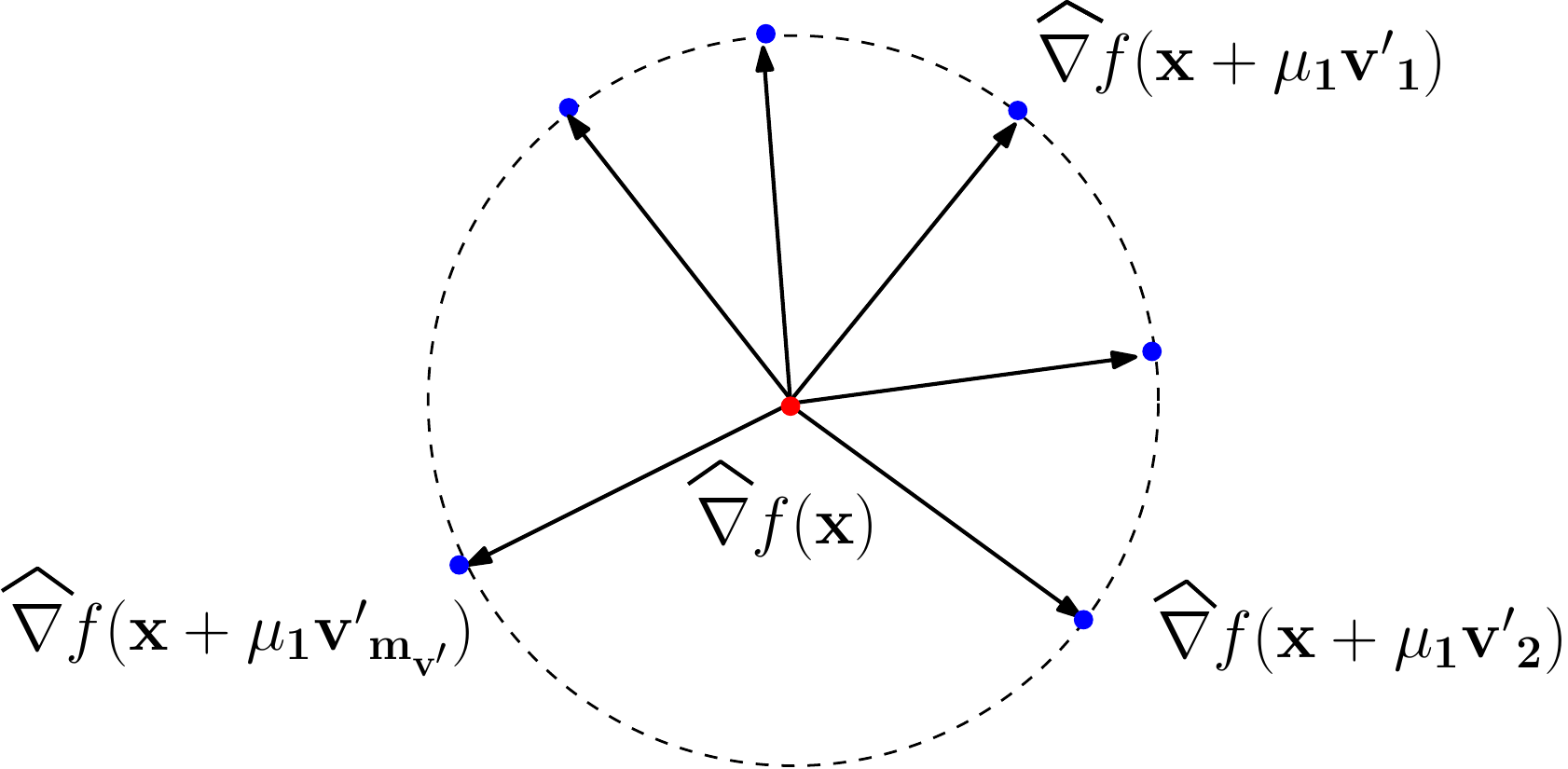}\label{fig:hess_samp}} \hspace{7mm}
     \subfloat[][]{\includegraphics[width=0.3\linewidth]{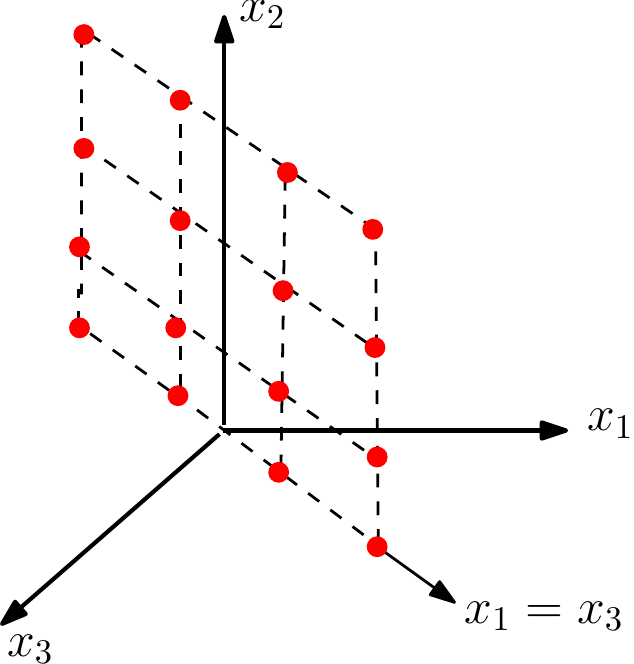}\label{fig:hash_samp}} 
     \caption{\small (a) $\hess f(\vecx)$ estimated using: $\est{\grad}f(\vecx)$ (at red disk) and 
		  neighborhood gradient estimates (at blue disks)
	    (b) Geometric picture: $\dimn = 3$, $\hashfn \in \calH_2^3$ with $\hashfn(1) = \hashfn(3) \neq \hashfn(2)$. 
	    Red disks are points in $\baseset(\hashfn)$. }
\end{figure}
Hash functions are widely used in theoretical computer science, such as in finding juntas \cite{Mossel03}. 
There exists a simple probabilistic method for constructing such a 
family, so that for any constant $C > 1$,  $\abs{\thashfam} \leq (C + 1)t e^t \log \dimn$ with high 
probability (w.h.p)\footnote{With probability $1-O(d^{-c})$ for some constant $c > 0$.} \cite{Devore2011}. 
For our purposes, we consider the family $\twohashfam$ so that for any distinct $i,j$, there exists 
$\hashfn \in \twohashfam$ such that $\hashfn(i) \neq \hashfn(j)$. 

For any $\hashfn \in \twohashfam$, let us now denote the vectors $\canvec_1(\hashfn), \canvec_2(\hashfn) \in \matR^{\dimn}$ where 
$$ (\canvec_i(\hashfn))_q = 
\left\{
	\begin{array}{ll}
		1  & \mbox{if } h(q) = i, \\
		0 & \mbox{otherwise,}
	\end{array}
\right.
$$
for $i=1,2$ and $q=1,\dots,\dimn$. 
Given at hand $\twohashfam$, we construct our set $\baseset$ using the 
procedure\footnote{Such sets were used in \cite{Devore2011} for a more general problem involving 
functions that are intrinsically $\totsparsity$ variate.} in \cite{Devore2011}. 
Specifically, for some $\numcen \in \matZ^{+}$, we construct for each $h \in \twohashfam$ the set:
\begin{small}
\begin{align*} 
\baseset(\hashfn) := \Bigg\{&\vecx(\hashfn) \in [-1,1]^{\dimn}: \vecx(\hashfn) = \sum_{i=1}^{2} c_i \canvec_i(\hashfn); \nonumber \\ &c_1,c_2 \in
\left\{-1,-\frac{\numcen-1}{\numcen},\dots,\frac{\numcen-1}{\numcen}, 1\right\}\Bigg\}.
\end{align*}
\end{small}
Then, we obtain $\baseset = \cup_{h \in \twohashfam} \baseset(h)$ as the set of points at 
which we will recover $\hess f$. Observe that $\baseset$ has the property of discretizing \emph{any}
$2$-dimensional canonical subspace, within $[-1,1]^{\dimn}$ with 
$\abs{\baseset} \leq (2\numcen+1)^2 \abs{\twohashfam} = O(\log \dimn)$.

\paragraph{Estimating sparse gradients.} Note that $\grad{f}$ is at most $\totsparsity$ sparse, 
due to the structure of $f$. We now describe the oracle that we use, 
for estimating sparse gradients.  As $f$ is $\calC^3$ smooth, therefore the third order 
Taylor's expansion of $f$ at $\vecx$, along $\vecv,-\vecv \in \matR^{\dimn}$, with step size
$\gradstep > 0$, and $\zeta = \vecx + \theta \vecv$, 
$\zeta^{\prime} = \vecx - \theta^{\prime} \vecv$; $\theta,\theta^{\prime} \in (0,\gradstep)$ leads to
\begin{align} \label{eq:taylor_exp_f}
&\frac{f(\vecx + \gradstep\vecv) - f(\vecx - \gradstep\vecv)}{2\gradstep} \nonumber \\ &= \dotprod{\vecv}{\grad f(\vecx)} 
+ (\thirdtayrem_3(\zeta) - \thirdtayrem_3(\zeta^{\prime}))/(2\gradstep).
\end{align}
\eqref{eq:taylor_exp_f} corresponds to a noisy-linear measurement of $\grad f(\vecx)$, with $\vecv$.
The ``noise'' here arises on account of the third order terms 
$\thirdtayrem_3(\zeta),\thirdtayrem_3(\zeta^{\prime}) = O(\gradstep^3)$, in the Taylor expansion.
Let $\calV$ denote the set of measurement vectors:
\begin{align*} 
\calV := \{v_j \in \matR^{\dimn} : v_{j,q} &= \pm 1/\sqrt{\numdirec} \ \text{w.p.} \ 1/2 \ \text{each};
 \nonumber \\  j&=1,\dots,\numdirec \ \text{and} \ q=1,\dots,{\dimn}\}. 
\end{align*}
Employing \eqref{eq:taylor_exp_f} at each $\vecv_j \in \calV$, we obtain: 
\begin{equation} \label{eq:cs_form}
\vecy = \matV\grad f(\vecx) + \taynoisvec. 
\end{equation}
Here, $\vecy \in \matR^{\numdirec}$ denotes the measurement vector with 
$(\vecy)_j = (f(\vecx + \gradstep\vecv_j) - f(\vecx - \gradstep\vecv_j))/(2\gradstep)$. 
Also, $\matV = [\vecv_1 \dots \vecv_{\numdirec}]^T \in \matR^{\numdirec \times {\dimn}}$ denotes the 
measurement matrix and $\taynoisvec \in \matR^{\numdirec}$ 
denotes the noise terms. We then estimate $\grad f(\vecx)$ via standard $\ell_1$ 
minimization\footnote{Can be solved efficiently using interior point methods \cite{Nesterov94}} \cite{Candes2006,Donoho2006,Wojta2012}. 
%
%
Estimating sparse gradients via CS, has been considered previously in \cite{Fornasier2010, Tyagi14_nips}, 
albeit using \emph{second order} Taylor expansions, for different function models. 

\subsection{Sampling scheme for estimating $\univsupp$} 
Having obtained an estimate $\est{\bivsupp}$ of $\bivsupp$ we now proceed to estimate $\univsupp$.  
Let $\est{\bivsuppvar}$ denote the set of variables in $\est{\bivsupp}$ and 
$\calP := [\dimn] \setminus \est{\bivsuppvar}$. Assuming $\est{\bivsupp} = \bivsupp$, 
we are now left with a SPAM on the \emph{reduced} variable set $\calP$. Consequently, we employ the 
sampling scheme of \cite{Tyagi14_nips}, 
wherein the gradient of $f$ is estimated at equispaced points, along a diagonal of $[-1,1]^{\dimn}$. 
For $\numcenpair \in \matZ^{+}$, this set is defined as:   
\begin{align*}
\baseset_{\text{diag}} := \Bigg\{\vecx &= (x \ x \ \cdots \ x) \in \matR^{\dimn}: \nonumber \\ x &\in \left\{-1,-\frac{\numcenpair-1}{\numcenpair},\dots,\frac{\numcenpair-1}{\numcenpair},1\right\}\Bigg\}.
\end{align*}
Note that $\abs{\baseset_{\text{diag}}} = 2\numcenpair+1$. The motivation for  estimating $\grad f$ at $\vecx \in \baseset_{\text{diag}}$ is 
that we obtain estimates of $\partial_p \phi_p$ at equispaced points within $[-1,1]$, for $p \in \univsupp$. With a sufficiently fine discretization, 
we would ``hit'' the critical regions associated with each $\partial_p \phi_p$, as defined in Assumption \ref{assum:actvar_iden}. 
By applying a thresholding operation, we would then be able to identify each $p \in \univsupp$.

To this end, consider the set of sampling directions: 
\begin{align*} 
\calVpp := \{\vecvpp_j \in \matR^{\dimn} : \vpp_{j,q} &= \pm 1/\sqrt{\numdirecpp} \ \text{w.p.} \ 1/2 \ \text{each};
 \nonumber \\  j&=1,\dots,\numdirecpp \ \text{and} \ q=1,\dots,{\dimn}\},  
\end{align*}
and let $\gradstepp > 0$ denote the step size. For each $\vecx \in \baseset_{\text{diag}}$, we will query $f$ at 
points: $(\vecx + \gradstepp \vecvpp_j)_{\calP}, (\vecx - \gradstepp \vecvpp_j)_{\calP}$; $\vecvpp_j \in \calVpp$, \emph{restricted} to $\calP$. 
Then, as described earlier, we can form a linear system consisting of $\numdirecpp$ equations, and solve it via $\ell_1$ minimization 
to obtain the gradient estimate.
The complete procedure for estimating $\univsupp,\bivsupp$, is described formally in Algorithm \ref{algo:gen_overlap}. 
%
\begin{algorithm*}[!ht]
\caption{Algorithm for estimating $\univsupp,\bivsupp$} \label{algo:gen_overlap} 
\begin{algorithmic}[1] 
\State \textbf{Input:} $\numdirec,\numdirecp, \numcen, \numcenpair \in \matZ^{+}$; $\gradstep, \hessstep, \gradstepp > 0$; 
$\hesssamperr > 0, \derivsamperrpp > 0$.
\State \textbf{Initialization:} $\est{\univsupp}, \est{\bivsupp} = \emptyset$.
\State \textbf{Output:} Estimates $\est{\bivsupp}$, $ \est{\univsupp}$. \\
\hrulefill
\State Construct $(\dimn,2)$-hash family $\twohashfam$ and sets $\calV,\calV^{\prime}$. \label{algover:s2_step_1}
\For{$\hashfn \in \twohashfam$}
 	\State Construct the set $\baseset(\hashfn)$. 
	\For {$i = 1,\dots,(2\numcen+1)^2$ and $\vecx_i \in \baseset(\hashfn)$} \label{algover:s2_step_2}
			\State $(\vecy_i)_j = \frac{f(\vecx_i + \gradstep \vecv_j)-f(\vecx_i - \gradstep \vecv_j)}{2\gradstep}$; $j=1,\dots,\numdirec$; $\vecv_j \in \calV$. \label{algover:s2_query_1}
		\State $\est{\grad} f(\vecx_i) := \argmin{\vecy_i = \matV\vecz} \norm{\vecz}_1$. \label{algover:s2_grad_base}
		\For{$p = 1,\dots,\numdirecp$} 
			\State $(\vecy_{i,p})_j = \frac{f(\vecx_i + \hessstep\vecvp_p + \gradstep \vecv_j)-f(\vecx_i + \hessstep\vecvp_p - \gradstep \vecv_j)}{2\gradstep}$; \label{algover:s2_query_2}
			$j=1,\dots,\numdirec$; $\vecvp_p \in \calVp$.  \hfill \textsc{Estimation of } $\bivsupp$
			\State $\est{\grad} f(\vecx_i + \hessstep\vecvp_p) := \argmin{\vecy_{i,p} = \matV \vecz} \norm{\vecz}_1$.\label{algover:s2_grad_1}	
		\EndFor
		\For{$q = 1,\dots,\dimn$}
			\State $(\vecy_q)_j = \frac{(\est{\grad} f(\vecx_i + \hessstep\vecvp_j) - \est{\grad} f(\vecx_i))_q}{\hessstep}$; 
			$j=1,\dots,\numdirecp$. \label{algover:s2_grad_2}
			\State $\est{\grad} \partial_q f(\vecx_i) := \argmin{\vecy_q = \matV^{\prime} \vecz} \norm{\vecz}_1$.  \label{algover:s2_grad_hess_row}
			\State $\est{\bivsupp} = \est{\bivsupp} \cup \set{(q,q^{\prime}) : q^{\prime} \in \set{q+1,\dots,d} \ \& \ \abs{(\est{\grad} \partial_q f(\vecx_i))_{q^{\prime}}} > \hesssamperr}$.
		\EndFor
	\EndFor 
\EndFor \\
\hrulefill
\State Construct the sets $\baseset_{\text{diag}}, \calVpp$ and initialize $\calP := [\dimn] \setminus \est{\bivsuppvar}$.
\For {$i=1,\dots,(2\numcenpair+1)$ and $\vecx_i \in \baseset_{\text{diag}}$} \label{algover:s1_step} 
		\State $(\vecy_i)_j = \frac{f((\vecx_i + \gradstepp \vecvpp_j)_{\calP})-f((\vecx_i - \gradstepp \vecvpp_j)_{\calP})}{2\gradstepp}$; 
		$j=1,\dots,\numdirecpp$; $\vecv_j \in \calVpp$. \label{algover:s1_grad}
		\State $(\est{\grad} f((\vecx_i)_{\calP}))_{\calP} := \argmin{\vecy_i = (\matVpp)_{\calP}(\vecz)_{\calP}} \norm{(\vecz)_{\calP}}_1$. \hfill 
		\textsc{Estimation of } $\univsupp$
		\State $\est{\univsupp} = \est{\univsupp} \cup \set{q \in \calP : \abs{((\est{\grad} f((\vecx_i)_{\calP})_q} > \derivsamperrpp}.$
\EndFor
\end{algorithmic}
\end{algorithm*}

\section{Theoretical guarantees for noiseless case} \label{sec:noiseless_query_res}
Next, we provide sufficient conditions on our sampling parameters that guarantee exact recovery 
of $\univsupp, \bivsupp$, in the noiseless query setting. 
This is stated in the following Theorem. All proofs are deferred to the appendix. 
%
\begin{theorem} \label{thm:gen_overlap}
$\exists$ positive constants $\{c_i^{\prime}\}_{i=1}^{3}, \{C_i\}_{i=1}^{3}$ so that if: $\numcen \geq \critintmeas_2^{-1},$
$\numdirec > c_1^{\prime} \totsparsity \log\left(\dimn/\totsparsity\right),$ and 
$\numdirecp > c_2^{\prime} \maxdegree \log(\dimn/\maxdegree),$ then the following holds. 
Denoting $a = \frac{(4\maxdegree+1)\smconst_3}{2\sqrt{\numdirecp}}$, 
$b = \frac{C_1\sqrt{\numdirecp}((4\maxdegree+1)\totsparsity)\smconst_3}{3\numdirec}$, 
$a^{\prime} = \frac{\idenconst_2}{4aC_2}$, let $\gradstep, \hessstep$ satisfy: 
%
$\gradstep^2 < ({a^{\prime}}^2 a)/b$ and  $$\quad \hessstep \in (a^{\prime} - \sqrt{{a^{\prime}}^2 - (b\gradstep^2/a)}, 
a^{\prime} + \sqrt{{a^{\prime}}^2 - (b\gradstep^2/a)}).$$ 
%
We then have for  
$\hesssamperr = C_2 (a\hessstep + \frac{b\gradstep^2}{\hessstep})$, 
that $\est{\bivsupp} = \bivsupp$ w.h.p.
Provided $\est{\bivsupp} = \bivsupp$, if $\numcenpair \geq \critintmeas_1^{-1},$
$\numdirecpp > c_3^{\prime} (\totsparsity-\abs{\est{\bivsuppvar}}) \log(\frac{\abs{\calP}}{\totsparsity-\abs{\est{\bivsuppvar}}})$ and 
${\gradstepp}^2 < \frac{3\numdirecpp \idenconst_1}{C_3 (\totsparsity-\abs{\est{\bivsuppvar}}) \smconst_3},$ then  
$\derivsamperrpp = \frac{C_3 (\totsparsity-\abs{\est{\bivsuppvar}}) {\gradstepp}^2 \smconst_3}{6\numdirecpp}$, implies 
$\est{\univsupp} = \univsupp$ w.h.p.
\end{theorem}
\begin{remark} 
We note that the condition on $\gradstepp$ is less strict than in \cite{Tyagi14_nips} for identifying 
$\univsupp$. This is because in \cite{Tyagi14_nips}, the gradient is estimated via a forward difference procedure, 
while we perform a central difference procedure in \eqref{eq:taylor_exp_f}.
\end{remark}
\paragraph{Query complexity.} Estimating $\grad f(\vecx)$ at some fixed $\vecx$ requires 
$2\numdirec = O(\totsparsity\log \dimn)$ queries. Estimating $\hess f(\vecx)$ involves computing an 
additional $\numdirecp = O(\maxdegree \log \dimn)$ gradient vectors in a neighborhood of 
$\vecx$ -- implying $O(\numdirec\numdirecp) = O(\totsparsity\maxdegree(\log \dimn)^2)$ point queries. 
This consequently implies a total query complexity of 
$O(\totsparsity\maxdegree(\log \dimn)^2 \abs{\baseset}) = O(\critintmeas_2^{-2}\totsparsity\maxdegree(\log \dimn)^3)$, 
for estimating $\bivsupp$. 
We make an additional $O(\critintmeas_1^{-1} (\totsparsity-\abs{\est{\bivsuppvar}}) \log (\dimn - \abs{\est{\bivsuppvar}}))$ queries of $f$, 
in order to estimate $\univsupp$. Therefore, the overall query complexity for estimating 
$\univsupp,\bivsupp$ is $O(\critintmeas_2^{-2}\totsparsity\maxdegree(\log \dimn)^3)$.

$\twohashfam$ can be constructed in $\text{poly}(d)$ time. For each $\vecx \in \baseset$, we first solve $\numdirecp + 1$ 
linear programs (Steps \ref{algover:s2_grad_base}, \ref{algover:s2_grad_1}), each solvable in $\text{poly}(\numdirec, d)$ time. 
We then solve $d$ linear programs (Step \ref{algover:s2_grad_hess_row}), with each taking $\text{poly}(\numdirecp, d)$ time. 
This is done at $\abs{\baseset} = O(\critintmeas_2^{-2} \log d)$ points, hence the 
overall \emph{computation cost} for estimation of $\bivsupp$ (and later $\univsupp$) is polynomial in: the number of queries, and $d$. 
Lastly, we note that \cite{Bandeira12} also estimates sparse Hessians via CS, albeit for the function optimization problem. 
Their scheme entails a sample complexity\footnote{See \cite[Corollary $4.1$]{Bandeira12}} of 
$O(\totsparsity\maxdegree (\log(\totsparsity\maxdegree))^2 (\log d)^2)$ for estimating $\hess f(\vecx)$; 
this is worse by a $O((\log(\totsparsity\maxdegree))^2)$ term compared to our method.

\paragraph{Recovering the components of the model.} Having estimated $\univsupp, \bivsupp$, 
we can now estimate each underlying component in \eqref{eq:unique_mod_rep} by sampling $f$ along the 
\emph{subspace} corresponding to the component. Using these samples, one can then construct via standard techniques, 
a spline based quasi interpolant \cite{deBoor78} that \emph{uniformly} approximates 
the component. This is shown formally in the appendix.
%
\section{Impact of noise} \label{sec:noise_impact}
We now consider the case where the point queries are corrupted with external noise. This means that at query $\vecx$, 
we observe $f(\vecx) + \exnoisep$, where $\exnoisep \in \matR$ denotes external noise. 

In order to estimate $\grad f(\vecx)$, 
we obtain the samples : $f(\vecx + \gradstep \vecv_j) + \exnoisep_{j,1}$ and $f(\vecx - \gradstep \vecv_j) + \exnoisep_{j,2}$; 
$j = 1,\dots,\numdirec$. 
This changes \eqref{eq:cs_form} to the linear system $\vecy = \matV\grad f(\vecx) + \taynoisvec + \exnoisevec$, where 
$\exnoise_{j} = (\exnoisep_{j,1} - \exnoisep_{j,2})/(2\gradstep)$. 
Hence, the step-size $\gradstep$ needs to be chosen carefully now
 -- a small value would blow up the external noise component, while a large value would increase perturbation
 due to the higher order Taylor's terms.
\paragraph{Arbitrary bounded noise.} In this scenario, we assume the external noise to be arbitrary and bounded, meaning 
that $\abs{\exnoisep} < \exnoisemag$, for some finite $\exnoisemag \geq 0$. If $\exnoisemag$ is too large, 
then we would expect recovery 
of $\univsupp, \bivsupp$ to be impossible as the structure of $f$ would be destroyed. 

We show in Theorem \ref{thm:gen_overlap_arbnois} 
that if $\exnoisemag < \exnoisemag_1 = O\left(\nicefrac{\idenconst_2^{3}}{(\smconst_3^2 \maxdegree^{2} \sqrt{\totsparsity})}\right)$, 
then Algorithm \ref{algo:gen_overlap} recovers $\bivsupp$ with 
appropriate choice of sampling parameters. Furthermore, assuming $\bivsupp$ is recovered exactly, 
and provided $\exnoisemag$ additionally  
satisfies $\exnoisemag < \exnoisemag_2 = O\left(\nicefrac{\idenconst_1^{3/2}}{\sqrt{(\totsparsity-\abs{\est{\bivsuppvar}})\smconst_3}}\right)$, 
then the algorithm also recovers $\univsupp$ exactly. 
In contrast to Theorem \ref{thm:gen_overlap}, the step size $\gradstep$ cannot be chosen arbitrarily small now, 
due to external noise. 
\begin{theorem} \label{thm:gen_overlap_arbnois}
Let $\numcen, \numcenpair, \numdirec, \numdirecp, \numdirecpp$ be as defined in 
Theorem \ref{thm:gen_overlap}. Say $\exnoisemag < \exnoisemag_1 = O\left(\frac{\idenconst_2^{3}}{\smconst_3^2 \maxdegree^{2} \sqrt{\totsparsity}}\right).$
Denoting $b^{\prime} = 2C_1\sqrt{\numdirec\numdirecp}$, $\exists 0 < A_1 < A_2$ and $0 < A_3 < A_4$ so that for $\gradstep \in (A_1,A_2)$, 
$\hessstep \in (A_3, A_4)$ and $\hesssamperr = C_2 (a\hessstep+ \frac{b \gradstep^2}{\hessstep} + \frac{b^{\prime}\exnoisemag}{\gradstep\hessstep}),$
we have $\est{\bivsupp} = \bivsupp$ w.h.p. 
Given $\est{\bivsupp} = \bivsupp$, 
denote $a_1 = \nicefrac{(\totsparsity-\abs{\est{\bivsuppvar}}) \smconst_3}{(6\numdirecpp)}$, $b_1 = \sqrt{\numdirecpp}$ 
and say $\exnoisemag < \exnoisemag_2 =  O\left(\frac{\idenconst_1^{3/2}}{\sqrt{(\totsparsity-\abs{\est{\bivsuppvar}})\smconst_3}}\right).$
$\exists 0<A_5<A_6$ so that $\gradstepp \in (A_5,A_6)$, $\derivsamperrpp = C_3(a_1 {\gradstepp}^2 + \frac{b_1\exnoisemag}{\gradstepp})$
implies $\est{\univsupp} = \univsupp$ w.h.p.
\end{theorem} 
\paragraph{Stochastic noise.} We now assume the point queries to be corrupted with i.i.d. Gaussian noise, so that $\exnoisep \sim \calN(0,\sigma^2)$ 
with variance $\sigma^2$. We consider resampling each point query a sufficient number of times, and averaging the 
values. During the $\bivsupp$ estimation phase, 
we resample each query $N_1$ times so that $\exnoisep \sim \calN(0,\sigma^2/N_1)$. 
For any $0 < \exnoisemag < \exnoisemag_1$, if $N_1$ is suitably large, then we can uniformly bound $\abs{\exnoisep} < \exnoisemag$ -- via standard tail 
bounds for Gaussians -- over all noise samples, with high probability.
Consequently, we can use the result of Theorem \ref{thm:gen_overlap_arbnois} for estimating $\bivsupp$. 
The same reasoning applies to Step \ref{algover:s1_grad}, \textit{i.e.}, the $\univsupp$ estimation phase, 
where we resample each query $N_2$ times. 
\begin{theorem} \label{thm:gen_overlap_gaussnois}
Let $\numcen, \numcenpair, \numdirec, \numdirecp, \numdirecpp$ be as defined in 
Theorem \ref{thm:gen_overlap}. For any $\exnoisemag < \exnoisemag_1$, $0 < p_1 < 1$, 
say we resample each query in Steps \ref{algover:s2_query_1}, \ref{algover:s2_query_2}, 
$N_1 > \frac{\sigma^2}{\exnoisemag^2} \log (\frac{\sqrt{2} \sigma}{\exnoisemag p_1}\numdirec(\numdirecp+1)(2\numcen+1)^2\abs{\twohashfam})$
times, and take the average. For $\gradstep, \hessstep, \hesssamperr$ as in Theorem \ref{thm:gen_overlap_arbnois}, 
we have $\est{\bivsupp} = \bivsupp$ with probability at least $1 -p_1 - o(1)$. 
Given $\est{\bivsupp} = \bivsupp$, with $\exnoisemagp < \exnoisemag_2$, $0 < p_2 < 1$, 
say we resample each query in Step \ref{algover:s1_grad},  
$N_2 > \frac{\sigma^2}{{\exnoisemagp}^2} \log(\frac{\sqrt{2} \sigma (2\numcenpair+1)\numdirecpp}{\exnoisemagp p_2})$ times, and take the average. 
Then for $\gradstepp, \derivsamperrpp$ as in Theorem \ref{thm:gen_overlap_arbnois} (with $\exnoisemag$ replaced by $\exnoisemagp$), 
we have $\est{\univsupp} = \univsupp$ with probability at least $1 - p_2 - o(1)$. 
\end{theorem}
\paragraph{Query complexity.} In the case of arbitrary, but bounded noise, the query complexity remains the same as for the noiseless case. 
In case of i.i.d. Gaussian noise, for estimating $\bivsupp$, we have $\exnoisemag = O(\maxdegree^{-2} \totsparsity^{-1/2})$. 
Choosing $p_1 = \dimn^{-\delta}$ 
for any constant $\delta > 0$ gives us $N_1 = O(\maxdegree^4 \totsparsity \log \dimn)$.
This means that with $O(N_1 \totsparsity\maxdegree(\log \dimn)^3 \abs{\baseset}) = O(\maxdegree^5 \totsparsity^2 (\log \dimn)^4 \critintmeas_2^{-2})$ 
queries, $\est{\bivsupp} = \bivsupp$ holds w.h.p.
Next, for estimating $\univsupp$, we have $\exnoisemagp = O((\totsparsity - \abs{\bivsuppvar})^{-1/2})$. 
Choosing $p_2 = ((\dimn - \abs{\bivsuppvar})^{-\delta})$ 
for any constant $\delta > 0$, we get $N_2 = O((\totsparsity - \abs{\bivsuppvar}) \log (\dimn - \abs{\bivsuppvar}))$. 
This means the query complexity for estimating $\univsupp$ is 
$O(N_2 \critintmeas_1^{-1} (\totsparsity-\abs{\est{\bivsuppvar}}) \log (\dimn - \abs{\est{\bivsuppvar}})) = 
O(\critintmeas_1^{-1} (\totsparsity-\abs{\est{\bivsuppvar}})^2 (\log (\dimn - \abs{\est{\bivsuppvar}}))^2)$.
Therefore, the overall query complexity of Algorithm \ref{algo:gen_overlap} 
for estimating $\univsupp, \bivsupp$ is $O(\maxdegree^5 \totsparsity^2 (\log \dimn)^4 \critintmeas_2^{-2})$. 
\begin{remark}
We saw above that $O(\totsparsity^2 (\log \dimn)^2)$ samples are sufficient for estimating $\univsupp$ 
in presence of i.i.d Gaussian noise. This improves the corresponding bound in \cite{Tyagi14_nips} 
by a $O(\totsparsity)$ factor, and is due to the less strict condition on $\gradstepp$.
\end{remark}
\paragraph{Recovering the components of the model.} 
Having identified $\univsupp, \bivsupp$, we can estimate the 
underlying components in \eqref{eq:unique_mod_rep}, via standard nonparametric regression for ANOVA type models 
\cite{stone94}. Alternately, for each component, we could also sample $f$ along the subspace 
corresponding to the component and then perform regression, to obtain its estimate with 
\emph{uniform} error bounds. This is shown formally in the appendix.

\section{Related work}

\paragraph{Learning SPAMs.} We begin with an overview of results for learning SPAMs, in the regression setting.
\cite{Lin2006} proposed the COSSO algorithm, that extends the Lasso to the reproducing kernel Hilbert space 
(RKHS) setting. \cite{Yuan07anova} generalizes the non negative garrote to the nonparametric setting. 
\cite{Koltch08, Ravi2009, Meier2009} consider least squares methods, regularized by sparsity inducing penalty terms, for learning such models.
\cite{Raskutti2012,Koltch2010} propose a convex program for estimating $f$ (in the RKHS setting) that achieves the minimax optimal 
error rates. \cite{Huang2010} proposes a method based on the adaptive group Lasso. 
These methods are designed for learning SPAMs and cannot handle models of the form \eqref{eq:intro_gspam_form}.

\paragraph{Learning generalized SPAMs.} There exist fewer results for generalized SPAMs of the form \eqref{eq:intro_gspam_form}, 
in the regression setting. The COSSO algorithm \cite{Lin2006} can handle \eqref{eq:intro_gspam_form}, 
however its convergence rates are shown only for the 
case of no interactions. \cite{Rad2010} proposes the VANISH 
algorithm -- a least squares method with sparsity constraints. 
It is shown to be sparsistent, \textit{i.e.}, it asymptotically recovers $\univsupp,\bivsupp$ for $n \rightarrow \infty$. 
They also show a consistency result for estimating $f$, similar to \cite{Ravi2009}. 
\cite{Storlie2011} proposes the ACOSSO method, an adaptive version of the COSSO algorithm, 
which can also handle \eqref{eq:intro_gspam_form}. 
They derive convergence rates and sparsistency results for their method, albeit for the case of no interactions.
\cite{Dala2014} studies a generalization of \eqref{eq:intro_gspam_form} that allows for the presence of 
a sparse number of $m$-wise interaction terms for some additional sparsity parameter $m$. 
While they derive\footnote{In the Gaussian white noise model, which is known to be asymptically equivalent to the regression model as $n \rightarrow \infty$.} 
non-asymptotic $L_2$ error rates for estimating $f$, they do not guarantee unique identification of 
the interaction terms for any value of $m$. 
A special case of \eqref{eq:intro_gspam_form} -- where $\phi_p$'s are linear and each $\phi_{\lpair}$ is of the form $x_l x_{\lp}$ --
has been studied considerably. Within this setting, there exist algorithms that recover $\univsupp,\bivsupp$, along with 
convergence rates for estimating $f$, but only in the limit of large $n$ \cite{Choi2010,Rad2010,Bien2013}. 
\cite{Nazer2010} generalized this to the setting of sparse multilinear systems -- albeit in the noiseless setting -- 
and derived non-asymptotic sampling bounds for identifying the interaction terms. 
However finite sample bounds for the non-linear model \eqref{eq:intro_gspam_form} are not known in general.

\paragraph{Learning generic low-dimensional function models.} There exists related work in approximation theory -- 
which is also the setting considered in this paper -- wherein one assumes 
freedom to query $f$ at any desired set of points within its domain. 
\cite{Devore2011} considers functions depending on an unknown subset $\totsupp$ ($\abs{\totsupp} = \totsparsity$) 
of the variables -- a more general model than \eqref{eq:intro_gspam_form}. 
They provide a choice of query points of size $O(c^k \totsparsity \log d)$ for some constant $c > 1$, 
and algorithms that recover $\totsupp$ w.h.p. \cite{karin2011} 
derives a simpler algorithm with sample complexity  
$O((C_1^{4}/\alpha^4) \totsparsity (\log \dimn)^2)$ for recovering $\totsupp$ w.h.p., 
where $C_1,\alpha$ depend\footnote{$C_1 = \max_{i \in \totsupp} \norm{\partial_i f}_{\infty}$ and 
$\alpha = \min_{i \in \totsupp} \norm{\partial_i f}_1$} on smoothness of $f$. For general $\totsparsity$-variate $f$: 
$\alpha = c^{-\totsparsity}$ for some constant $c > 1$, while for our model \eqref{eq:intro_gspam_form}: $C_1 = O(\maxdegree)$. 
This model was also studied in \cite{Comming2011,Comming2012} in the regression setting -- they proposed an estimator that 
recovers $\totsupp$ w.h.p, with sample complexity $O(c^{\totsparsity} \totsparsity \log d)$.
\cite{Fornasier2010,Tyagi2012_nips} generalize this model to functions $f$ of the form 
$f(\vecx) = g(\matA\vecx)$, for unknown $\matA \in \matR^{k \times d}$. They derive algorithms that approximately 
recover the row-span of $\matA$ w.h.p, with sample complexities typically polynomial in $\dimn$. 

While the above methods could possibly recover $\totsupp$, they are not designed for identifying \emph{interactions} among the variables. 
Specifically, their sample complexities exhibit a worse dependence on $\totsparsity, \maxdegree$ and/or $d$. 
\cite{Tyagi14_nips} provides a sampling scheme that specifically learns SPAMs, with sample complexities  
$O(\totsparsity \log d)$, $O(\totsparsity^3 (\log \dimn)^2)$, in the absence/presence of Gaussian noise, respectively.
%
\section{Simulation results} \label{sec:sims}
\paragraph{Dependence on $\dimn$.} We first consider the following experimental setup: $\univsupp =
\left\{1, 2 \right\}$ and $\bivsupp = \{(3, 4), (4, 5)\}$,
which implies $\univsparsity = 2$, $\bivsparsity = 2$, $\maxdegree = 2$ and $\totsparsity = 5$. 
We consider two models:  \vspace{-4mm}
\begin{itemize}
\item [$(i)$] $f_1(\vecx) = 2x_1 - 3x_2^2 + 4x_3x_4 - 5x_4x_5$, \vspace{-2mm}
\item [$(ii)$] $f_2(\vecx) = 10 \sin(\pi \cdot x_1) + 5 e^{-2x_2} + 10\sin(\pi \cdot x_3 x_4) + 5 e^{-2x_4 x_5}$. 
\end{itemize} \vspace{-5mm}
\begin{figure}[!ht]
\begin{center}
   \includegraphics[width=0.52\columnwidth]{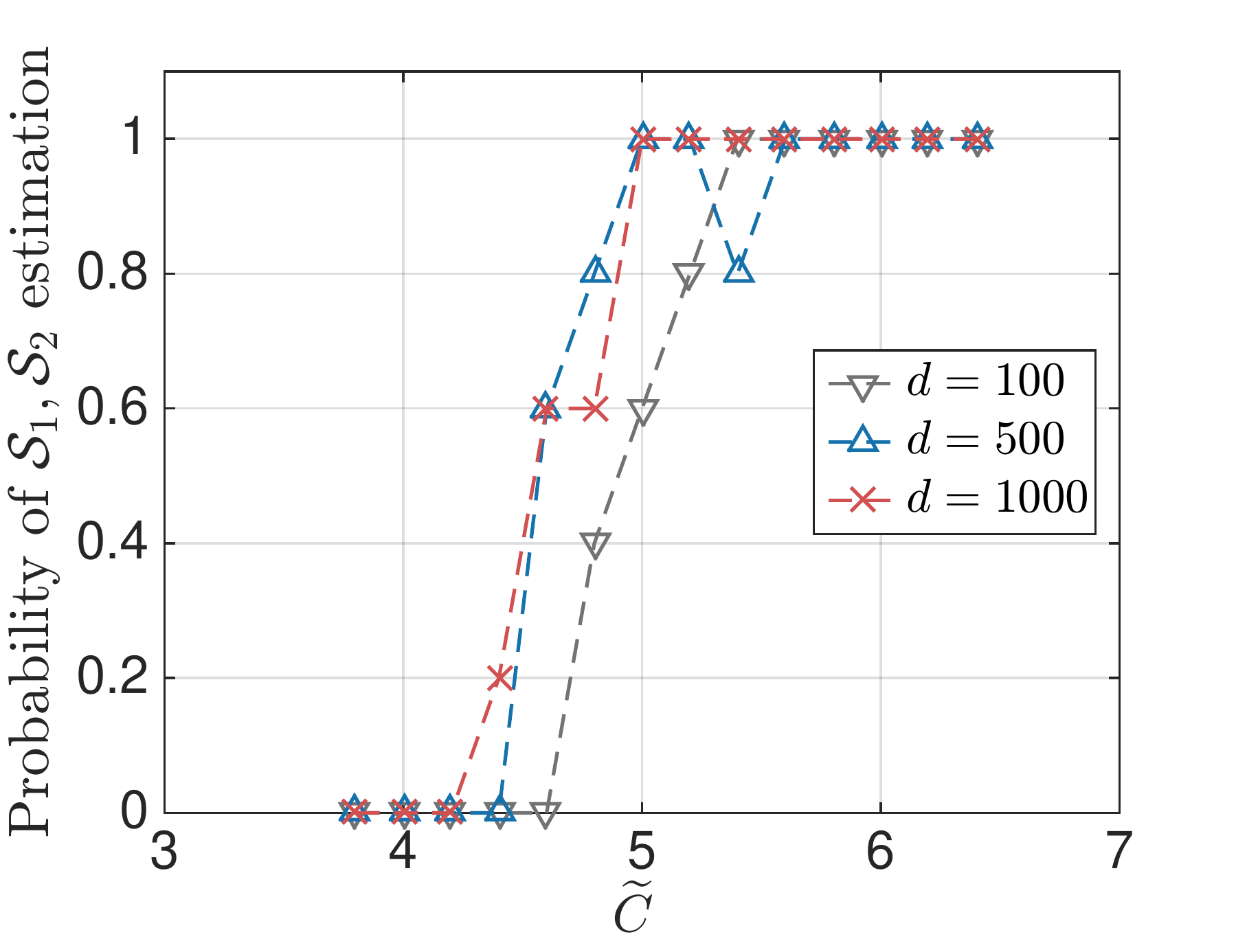} \hspace{-5mm}
   \includegraphics[width=0.52\columnwidth]{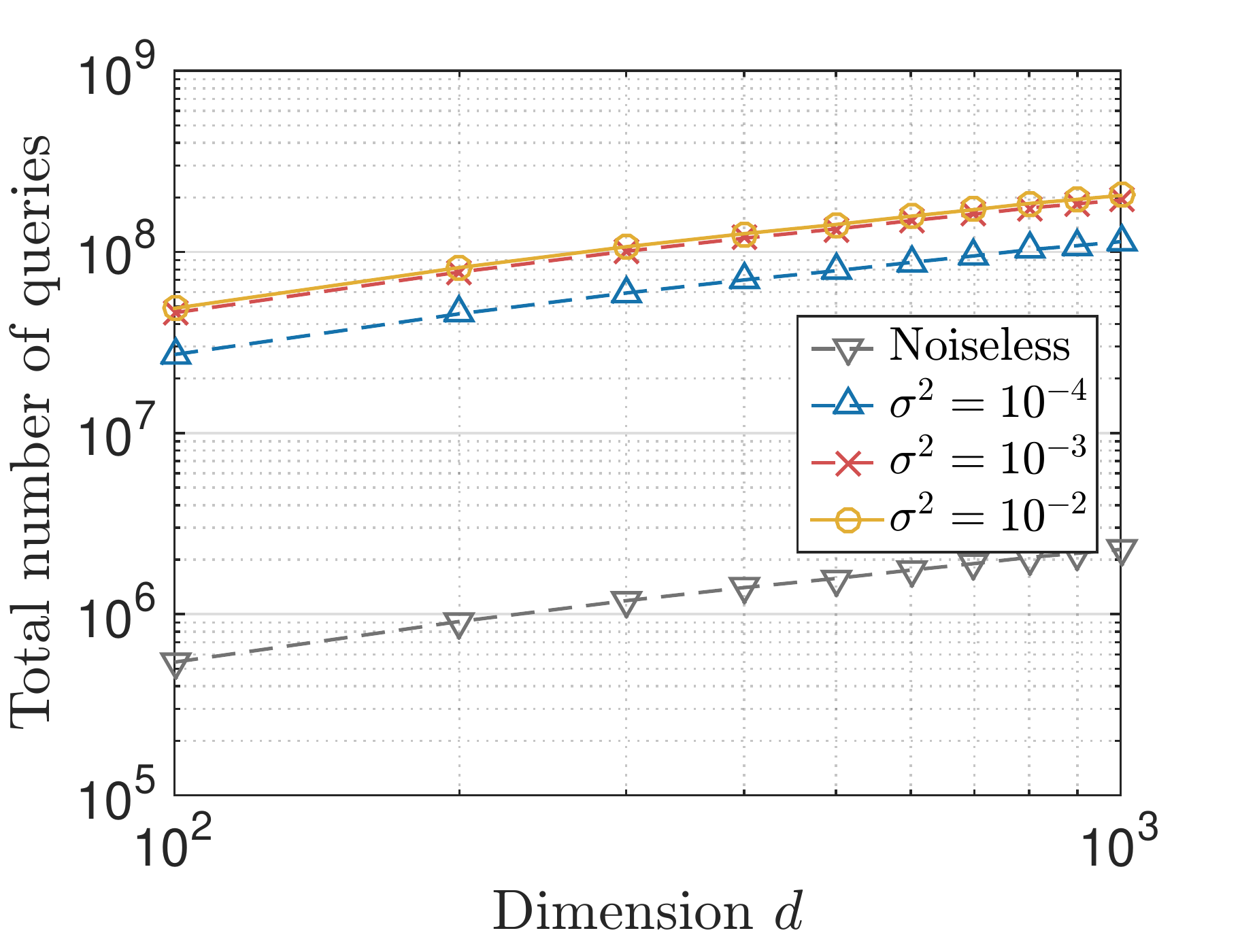} \\
	  \includegraphics[width=0.52\columnwidth]{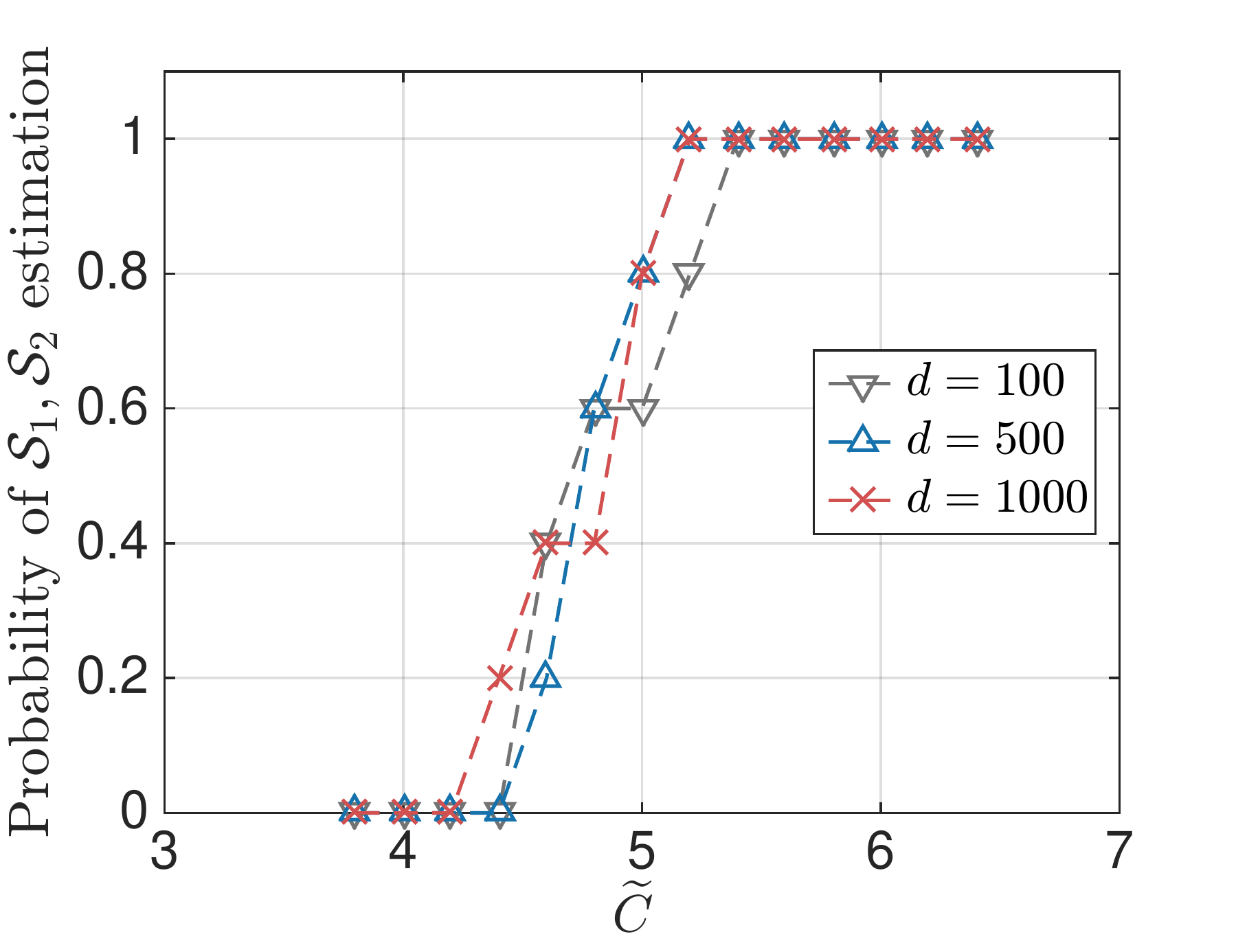} \hspace{-5mm}
  \includegraphics[width=0.52\columnwidth]{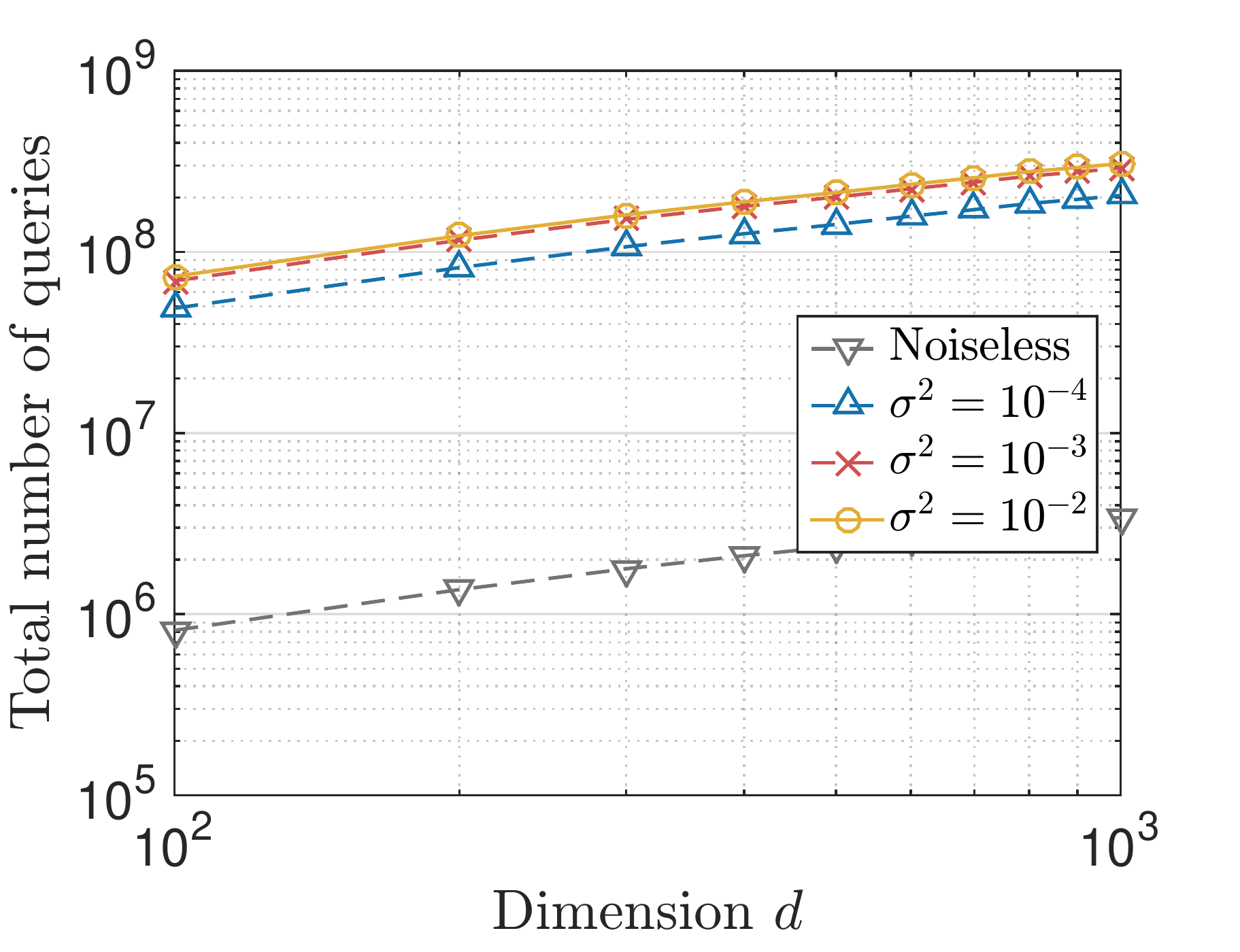} 
\end{center} \vspace{-3mm}
\caption{\small First (resp. second) row is for $f_1$ (resp. $f_2$). Left panel depicts
the success probability of identifying \empty{exactly} $\univsupp,\bivsupp$, in the noiseless case. 
$x$-axis represent the constant $\widetilde{C}$. The right panel depicts 
total queries vs. $\dimn$ for exact recovery, with $\widetilde{C} = 5.6$ and various noise settings.
All results are over $5$ independent Monte Carlo trials.} \label{exp:f1_f2_plots}
\end{figure}
We begin with the relatively simple model $f_1$, for which the problem parameters are set to: 
$\critintmeas_1 = 0.3$, $\critintmeas_2 = 1$, $\idenconst_1 = 2$, $\idenconst_2 = 3$, $\smconst_3 = 6$. 
We obtain $\numcen = 1$, $\numcenpair = 4$. We use the same constant $\widetilde{C}$ 
when we set $\numdirec := \widetilde{C} \totsparsity \log\left(\dimn/\totsparsity\right)$, 
$\numdirecp := \widetilde{C} \maxdegree \log(\dimn/\maxdegree),$ and 
$\numdirecpp := \widetilde{C}  (\totsparsity-\abs{\est{\bivsuppvar}}) \log(\frac{\abs{\calP}}{\totsparsity-\abs{\est{\bivsuppvar}}})$. 
For the construction of the hash functions, we set the size to $|\twohashfam| = C^{\prime} \log d$ with $C^{\prime} = 1.7$, 
leading to $|\mathcal{H}_2^d| \in [8, 12]$ for $10^2 \leq d \leq 10^3$. 
We choose step sizes: $\gradstep, \hessstep, \gradstepp$ and thresholds: $\hesssamperr, \derivsamperrpp$ 
as in Theorem \ref{thm:gen_overlap_arbnois}. As CS solver, we use the ALPS 
algorithm \cite{kyrillidis2011recipes}, an efficient first-order method.

For the noisy setting, we consider the function values to be corrupted with i.i.d. Gaussian noise. 
The noise variance values considered are: $\sigma^2 \in \left\{10^{-4}, 10^{-3}, 10^{-2}\right\}$ for which we choose 
resampling factors: $(N_1,N_2) \in \set{(50,20), (85,36), (90,40)}$. 
We see in Fig. \ref{exp:f1_f2_plots}, that for $\widetilde{C} \approx 5.6$ the probability of successful identification (noiseless case)
undergoes a phase transition and becomes close to $1$, for different values of $\dimn$. This validates 
Theorem \ref{thm:gen_overlap}. Fixing $\widetilde{C} = 5.6$, we then see that with the total number of queries 
growing slowly with $\dimn$, we have successful identification. For the noisy case, the total number of queries is roughly 
$10^2$ times that in the noiseless setting, however the scaling with $\dimn$ is similar to the noiseless case. 

We next consider the relatively harder model: $f_2$, where the problem parameters are set to: 
$\critintmeas_1 = \critintmeas_2 = 0.3$, $\idenconst_1 = 8$, $\idenconst_2 = 4$, $\smconst_3 = 35$ and, 
$\numcen = \numcenpair = 4$. We see in Fig. \ref{exp:f1_f2_plots}, a phase transition (noiseless case) at 
$\widetilde{C} = 5.6$ thus validating Theorem \ref{thm:gen_overlap}. For noisy cases, we consider 
$\sigma^2$ as before, and $(N_1,N_2) \in \set{(60,30), (90,40), (95,43)}$. The number of queries is 
seen to be slightly larger than that for $f_1$.

\paragraph{Dependence on $\totsparsity$.} We now demonstrate the scaling of the total number of queries versus 
the sparsity $\totsparsity$ for identification of $\univsupp,\bivsupp$. Consider the model 
{
$f_3(\vecx) = \sum_{i = 1}^T (\alpha_1 \vecx_{(i-1)5 + 1} - \alpha_2\vecx_{(i-1)5 + 2}  ^2  
+ \alpha_3 \vecx_{(i-1)5 + 3} \vecx_{(i-1)5 + 4} - \alpha_4 \vecx_{(i-1)5 + 4} \vecx_{(i-1)5 + 5})$ 
} 
where $\vecx \in \matR^{\dimn}$ for $d  = 500$. Here, $\alpha_i \in [2, 5], \forall i$; \textit{i.e.}, 
we randomly selected $\alpha_i$'s within range and kept the values fixed for all $5$ Monte Carlo iterations. 
Note that $\maxdegree = 2$ and the sparsity $\totsparsity = 5T$; we consider $T \in \left\{1, 2, \dots, 10\right\}$. 
We set $\critintmeas_1 = 0.3$, $\critintmeas_2 = 1$, $\idenconst_1 = 2$, $\idenconst_2 = 3$, $\smconst_3 = 6$ and $\widetilde{C} = 5.6$. 
For the noisy cases, we consider $\sigma^2$ as before, and choose the same values for $(N_1,N_2)$ as for $f_1$. 
In Figure \ref{fig:exp_k} we see that the number of queries scales as $\sim \totsparsity \log(\dimn/\totsparsity)$, and 
is roughly $10^2$ more in the noisy case as compared to the noiseless setting.
\begin{figure}[!ht]
	\begin{center}
		\includegraphics[width=0.37\textwidth]{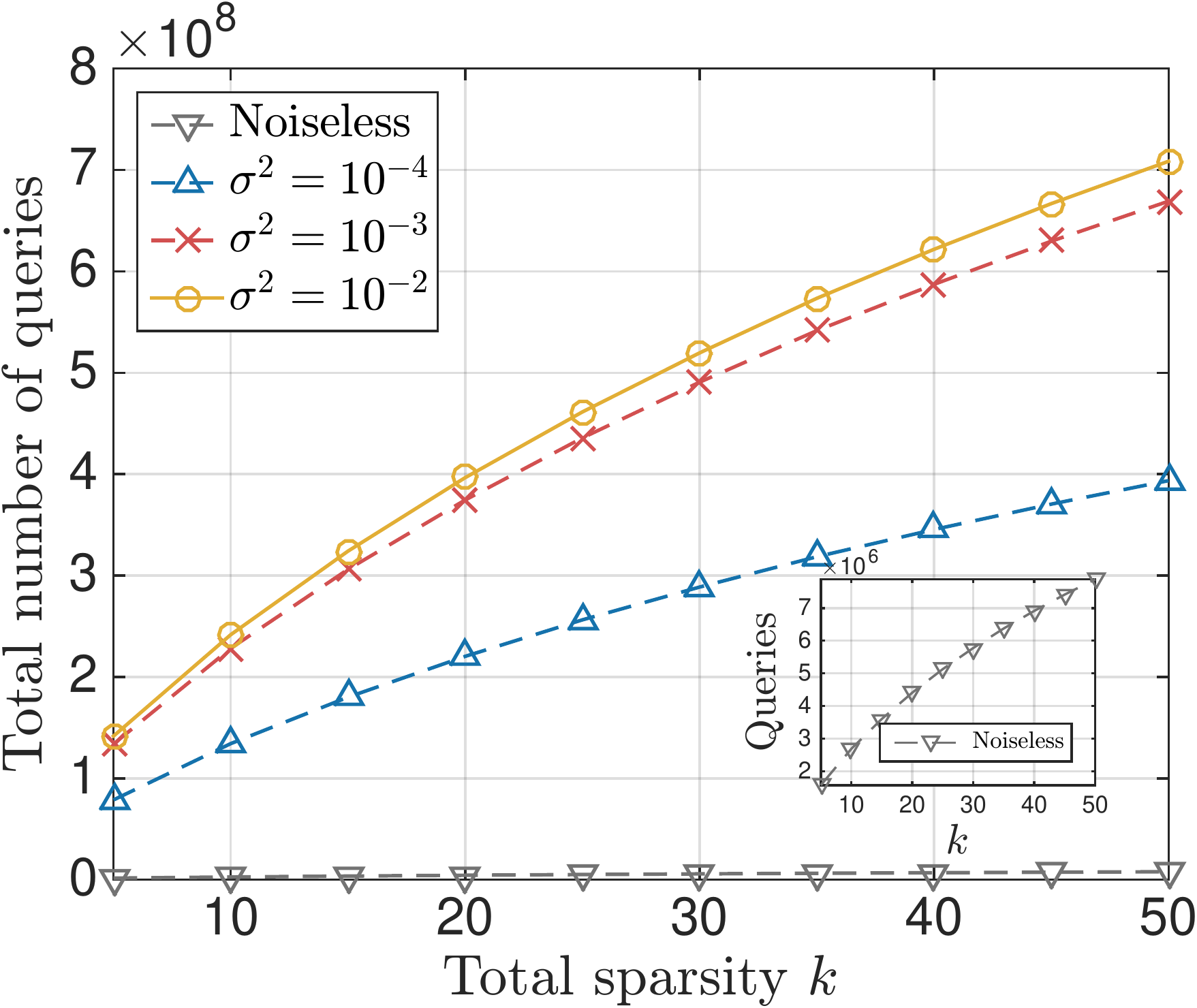} 
	\end{center} \vspace{-3mm}
	\caption{\small Total number of queries versus $\totsparsity$ for $f_3$. 
	This is shown for both noiseless and noisy cases (i.i.d Gaussian).} 
\label{fig:exp_k}
\end{figure}
\paragraph{Dependence on $\maxdegree$.}
We now demonstrate the scaling of the total queries versus 
the maximum degree $\maxdegree$ for identification of $\univsupp,\bivsupp$. Consider the model 
{
$f_4(\vecx) = \alpha_1 \vecx_{1} - \alpha_2\vecx_{2}^2 + \sum_{i = 1}^{T} (\alpha_{3,i} \vecx_{3} \vecx_{i+3}) 
+ \sum_{i=1}^{5} (\alpha_{4,i}\vecx_{2+2i}\vecx_{3+2i}).$ 
}
We choose $d = 500$, $\widetilde{C} = 6$, $\alpha_i \in [2,\dots, 5], \forall i$ (as earlier) and set  
$\critintmeas_1 = 0.3$, $\critintmeas_2 = 1$, $\idenconst_1 = 2$, $\idenconst_2 = 3$, $\smconst_3 = 6$.
For $T \geq 2$, we have $\maxdegree = T$; we choose $T \in \set{2,3,\dots,10}$. 
Also note that $\totsparsity = 13$ throughout. For the noisy cases, we consider $\sigma^2$ as before, and choose 
$(N_1,N_2) \in \set{(70,40), (90,50), (100,70)}$. In Figure \ref{fig:exp_rho}, 
we see that the number of queries scales as $\sim \maxdegree \log(\dimn/\maxdegree)$, and 
is roughly $10^2$ more in the noisy case as compared to the noiseless setting.
\begin{figure}[!ht]
	\begin{center}
		\includegraphics[width=0.37\textwidth]{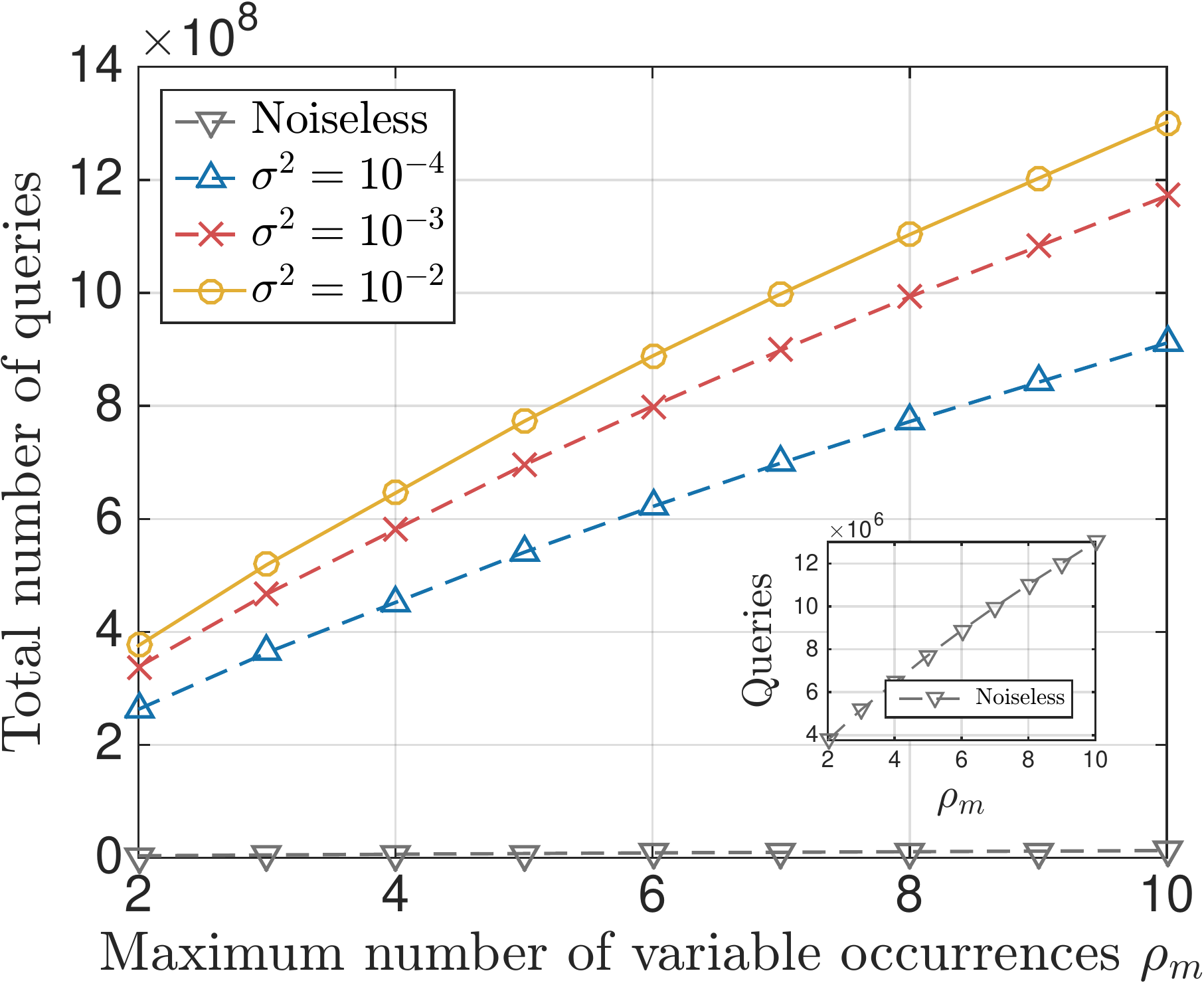} 
	\end{center} \vspace{-3mm}
	\caption{\small Total number of queries versus $\maxdegree$ 
	for $f_4$. 
	This is shown for both noiseless and noisy cases (i.i.d Gaussian).} 
	\label{fig:exp_rho}
\end{figure}
\section{Concluding remarks} \label{sec:concl_rems}
We proposed a sampling scheme for learning a generalized SPAM 
and provided finite sample bounds 
for recovering the underlying structure of such models. We also considered the setting where the point queries are corrupted 
with noise and analyzed sampling conditions for the same.  It would be interesting to improve the sampling bounds that we obtained, 
and under similar assumptions. We leave this for future work.
\paragraph{Acknowledgements.} This research was supported in part by SNSF grant CRSII$2$\_$147633$. 

\clearpage
\bibliographystyle{unsrt}

\bibliography{SPAM}

\newpage
\onecolumn
\appendix 
\setcounter{section}{0}
\setcounter{remark}{0}

\begin{centering}
\large \bf Supplementary Material : Learning Sparse Additive Models
with Interactions in High Dimensions.
\end{centering}

In this supplementary material, we prove the results stated in the paper. In Section \ref{sec:anova_uniq_rep}, 
we show that the model representation \eqref{eq:unique_mod_rep} is a unique representation for $f$ of the form \eqref{eq:gspam_form}.
In Section \ref{sec:proof}, we prove the main results of this paper namely: Theorem \ref{thm:gen_overlap}, 
Theorem \ref{thm:gen_overlap_arbnois} and Theorem \ref{thm:gen_overlap_gaussnois}. In Section \ref{sec:est_comp} 
we discuss how the individual components of the model \eqref{eq:unique_mod_rep} can be estimated once $\univsupp, \bivsupp$ are 
known. This is shown for both the noiseless as well as the noisy setting.

%
%
\section{Model uniqueness} \label{sec:anova_uniq_rep}

We show here that the model representation \eqref{eq:unique_mod_rep} is a unique representation for $f$ of the form \eqref{eq:gspam_form}.
We first note that any measurable $f:\matR^{\dimn} \rightarrow \matR$, admits a unique ANOVA decomposition (cf., \cite{Gu02}) of the form:
\begin{equation}
f(x_1,\dots,x_{\dimn}) = c + \sum_{\alpha} f_{\alpha}(x_{\alpha}) + \sum_{\alpha < \beta} f_{\alpha\beta} + \sum_{\alpha < \beta < \gamma} f_{\alpha\beta\gamma} + \cdots 
\end{equation}
Indeed, for any probability measure $\mu_{\alpha}$ on $\matR$; $\alpha = 1,\dots,\dimn$, 
let $\calE_{\alpha}$ denote the averaging operator, defined as
\begin{equation}
\calE_{\alpha}(f)(\vecx) := \int_{\matR} f(x_1,\dots,x_{\dimn}) d\mu_{\alpha}. 
\end{equation}
Then the components of the model can be written as : $c = (\prod_{\alpha} \calE_{\alpha}) f$, 
$f_{\alpha} = ((I-\calE_{\alpha})\prod_{\beta \neq \alpha} \calE_{\beta}) f$, 
$f_{\alpha\beta} = ((I-\calE_{\alpha})(I-\calE_{\beta})\prod_{\gamma \neq \alpha,\beta}\calE_{\gamma}) f$, and so on.
For our purpose, $\mu_{\alpha}$ is considered to be the uniform probability measure on $[-1,1]$. 
This is because we are interested in estimating $f$ within $[-1,1]^{\dimn}$. Given this, we now find the ANOVA decomposition of 
$f$ defined in \eqref{eq:gspam_form}. 

As a sanity check, let us verify that $f_{\alpha\beta\gamma} \equiv 0$ for all $\alpha < \beta < \gamma$. 
Indeed if $p \in \univsupp$, then at least two of $\alpha < \beta < \gamma$ will not be equal to $p$. Similarly for any $\lpair \in \bivsupp$, 
at least one of $\alpha,\beta,\gamma$ will not be equal to $l$ and $\lp$. This implies $f_{\alpha\beta\gamma} \equiv 0$. 
The same reasoning easily applies for high order components of the ANOVA decomposition. 

That $c = \expec[f] = \sum_{p \in \univsupp} \expec_p[\phi_p] + \sum_{\lpair \in \bivsupp} \expec_{\lpair}[\phi_{\lpair}]$ is readily seen. 
Next, we have that 
\begin{equation} \label{eq:first_ord_univ}
(I-\calE_{\alpha})\prod_{\beta \neq \alpha} \calE_{\beta} \phi_p = \left\{
\begin{array}{rl}
0 \quad ; & \alpha \neq p, \\
\phi_p - \expec_p[\phi_p] \quad ; & \alpha = p
\end{array} \right\}; \quad p \in \univsupp.
\end{equation}
\begin{equation} \label{eq:first_ord_biv}
(I-\calE_{\alpha})\prod_{\beta \neq \alpha} \calE_{\beta} \phi_{\lpair} = \left\{
\begin{array}{rl}
\expec_{\lp}[\phi_{\lpair}] - \expec_{\lpair}[\phi_{\lpair}] \quad ; & \alpha = l, \\
\expec_{l}[\phi_{\lpair}] - \expec_{\lpair}[\phi_{\lpair}] \quad ; & \alpha = \lp, \\
0 \quad ; & \alpha \neq l,\lp,
\end{array} \right\}; \quad \lpair \in \bivsupp.
\end{equation}
\eqref{eq:first_ord_univ}, \eqref{eq:first_ord_biv} give us the first order components of $\phi_p, \phi_{\lpair}$ respectively.
One can next verify, using the same arguments as earlier, that for any $\alpha < \beta$: 
\begin{equation} \label{eq:sec_ord_univ}
(I-\calE_{\alpha})(I-\calE_{\beta})\prod_{\gamma \neq \alpha,\beta} \calE_{\gamma} \phi_p = 0; \quad \forall p \in \univsupp. 
\end{equation}
Lastly, we have for any $\alpha < \beta$ that the corresponding second order component of $\phi_{\lpair}$ is given by:
\begin{equation} \label{eq:sec_ord_biv}
(I-\calE_{\alpha})(I-\calE_{\beta})\prod_{\gamma \neq \alpha,\beta} \calE_{\gamma} \phi_{\lpair} = \left\{
\begin{array}{rl}
\phi_{\lpair} - \expec_l[\phi_{\lpair}] \\ - \expec_{\lp}[\phi_{\lpair}] + \expec_{\lpair}[\phi_{\lpair}] \quad ; & \alpha = l, \beta = \lp, \\
0 \quad ; & \text{otherwise} 
\end{array} \right\}; \quad \lpair \in \bivsupp.
\end{equation}
We now make the following observations regarding the variables in $\univsupp \cap \bivsuppvar$.
\begin{enumerate}
\item For each $l \in \univsupp \cap \bivsuppvar$ such that: $\degree(l) = 1$, and $\lpair \in \bivsupp$, we can simply merge $\phi_l$ 
with $\phi_{\lpair}$. Thus $l$ is no longer in $\univsupp$. 

\item For each $l \in \univsupp \cap \bivsuppvar$ such that: $\degree(l) > 1$, we can add the first order component for $\phi_l$ 
with the total first order component corresponding to all $\phi_{\lpair}$'s and $\phi_{\lpairi}$'s. 
Hence again, $l$ will no longer be in $\univsupp$. 
\end{enumerate}
Therefore all $q \in \univsupp \cap \bivsuppvar$ can essentially be merged with $\bivsupp$. Keeping this re-arrangement 
in mind, we can to begin with, assume in \eqref{eq:gspam_form} that $\univsupp \cap \bivsuppvar = \emptyset$.
Then with the help of \eqref{eq:first_ord_univ}, \eqref{eq:first_ord_biv}, \eqref{eq:sec_ord_univ}, \eqref{eq:sec_ord_biv}, 
we have that any $f$ of the form \eqref{eq:gspam_form} (with $\univsupp \cap \bivsuppvar = \emptyset$), can be uniquely written as:
\begin{equation}
f(x_1,\dots,x_d) = c + \sum_{p \in \univsupp}\phitil_{p} (x_p) + \sum_{\lpair \in \bivsupp} \phitil_{\lpair} \xlpair + 
\sum_{q \in \bivsuppvar: \degree(q) > 1} \phitil_{q} (x_q); \quad \univsupp \cap \bivsuppvar = \emptyset,
\end{equation}
where
\begin{align}
c &= \sum_{p \in \univsupp} \expec_p[\phi_p] + \sum_{\lpair \in \bivsupp} \expec_{\lpair}[\phi_{\lpair}], \label{eq:mod_mean} \\
\phitil_{p} &= \phi_p - \expec_p[\phi_p]; \quad \forall p \in \univsupp, \label{eq:mod_s1}
\end{align}
\begin{equation} \label{eq:mod_s2_biv}
\phitil_{\lpair} = \left\{
\begin{array}{rl}
\phi_{\lpair} - \expec_{\lpair}[\phi_{\lpair}] ; & \degree(l), \degree(l^{\prime}) = 1, \\
\phi_{\lpair} - \expec_{l}[\phi_{\lpair}] ; &  \degree(l) = 1, \degree(l^{\prime}) > 1, \\
\phi_{\lpair} - \expec_{l^{\prime}}[\phi_{\lpair}] ; & \degree(l) > 1, \degree(l^{\prime}) = 1, \\
\phi_{\lpair} - \expec_{l}[\phi_{\lpair}] - \expec_{l^{\prime}}[\phi_{\lpair}] + \expec_{\lpair}[\phi_{\lpair}] ; & \degree(l) > 1, \degree(l^{\prime}) > 1,
\end{array} \right. 
\end{equation}
\begin{align} 
\text{and} \quad \phitil_{q} &= \sum_{q^{\prime}: (q,q^{\prime}) \in \bivsupp} (\expec_{q^{\prime}}[\phi_{(q,q^{\prime})}] - 
\expec_{(q,q^{\prime})}[\phi_{(q,q^{\prime})}]) \nonumber \\
&+ \sum_{q^{\prime}: (q^{\prime},q) \in \bivsupp} (\expec_{q^{\prime}}[\phi_{(q^{\prime},q)}] - 
\expec_{(q^{\prime},q)}[\phi_{(q^{\prime},q)}]); \quad \forall q \in \bivsuppvar: \degree(q) > 1. \label{eq:mod_s2_uni}
\end{align}
%
\section{Proofs}{\label{sec:proof}}
\subsection{Proof of Theorem \ref{thm:gen_overlap}} \label{subsec:proof_thm_genover}
The proof makes use of the following key theorem from \cite{Fornasier2010}, for stable approximation via $\ell_1$ minimization: 
$\triangle(\vecy) = \argmin{\vecy = \matV\vecz}\norm{\vecz}_1$. While the first part is standard (see for example
\cite{Baraniuk2008_simple}), 
the second result was stated in \cite{Fornasier2010} as a specialization of Theorem 1.2 
from \cite{Wojta2012} to the case of Bernoulli measurement matrices.
\begin{theorem}[\cite{Wojta2012,Fornasier2010}] \label{thm:sparse_recon_bound}
Let $\matV$ be a $\numdirec \times \dimn$ random matrix with all entries being Bernoulli i.i.d random variables scaled with $1/\sqrt{\numdirec}$. 
Then the following results hold.
\begin{enumerate}
\item Let $0 < \ripconst < 1$. Then there are two positive constants $c_1,c_2 > 0$, such that the matrix $\matV$ has the Restricted Isometry Property
\begin{equation}
(1-\ripconst) \norm{\vecw}_2^2 \leq \norm{\matV\vecw}_2^2 \leq (1+\ripconst) \norm{\vecw}_2^2 \label{eq:RIP}
\end{equation}
for all $\vecw \in \matR^{\dimn}$ such that $\#$supp($\vecw$) $\leq c_2 \numdirec / \log(\dimn/\numdirec)$ with probability at least $1-e^{-c_1 \numdirec}$.

\item Let us suppose $\dimn > (\log 6)^2 \numdirec$. Then there are positive constants $C, c_1^{\prime}, c_2^{\prime} > 0$ such that 
with probability at least $1 - e^{-c_1^{\prime} \numdirec} - e^{-\sqrt{\numdirec\dimn}}$ the matrix $\matV$ has the following property.
For every $\vecw \in \matR^{\dimn}$, $\taynoisvec \in \matR^{\numdirec}$ and every natural number 
$\totsparsity \leq c_2^{\prime} \numdirec / \log(\dimn/\numdirec)$, we have
\begin{equation} 
\norm{\triangle(\matV\vecw + \taynoisvec) - \vecw}_2 \leq C \left(\totsparsity^{-1/2} \sigma_{\totsparsity}(\vecw)_{1} + 
\max\set{\norm{\taynoisvec}_2, \sqrt{\log \dimn}\norm{\taynoisvec}_{\infty}}\right), \label{eq:sparse_recon_err}
\end{equation}
where
\begin{equation*}
\sigma_{\totsparsity}(\vecw)_{1} := \inf\set{\norm{\vecw - \vecz}_1 : \#\text{supp}(\vecz) \leq \totsparsity}
\end{equation*}
is the best $\totsparsity$-term approximation of $\vecw$.
\end{enumerate}
\end{theorem}
\begin{remark}
The proof of the second part of Theorem \ref{thm:sparse_recon_bound} requires \eqref{eq:RIP} to hold, which is the case in our setting with high probability. 
\end{remark}
\begin{remark} \label{rem:l1min_samp_bd}
Since $\numdirec \geq K$ is necessary, note that $K \leq c_2^{\prime} \numdirec / \log(d/\numdirec)$ is satisfied if 
$\numdirec > (1 / c_2^{\prime}) K \log(d/K)$. Also note that $K \log (d/K) > \log d$ in the 
regime\footnote{More precisely, if $d > K^{\frac{K}{K-1}}$.} $K \ll d$.
\end{remark}

We can now prove Theorem \ref{thm:gen_overlap}. 
The proof is divided into the following steps.
\paragraph{Bounding the $\hessestnoisb$ term.}
Since $\grad f(\vecx)$ is at most $\totsparsity$ sparse, therefore
for any $\vecx \in \matR^{\dimn}$ we immediately have from Theorem \ref{thm:sparse_recon_bound}, \eqref{eq:sparse_recon_err}, the following. 
$\exists C_1, c_4^{\prime}> 0, c_1^{\prime} \geq 1$ such that for 
$c_1^{\prime} \totsparsity \log(\frac{\dimn}{\totsparsity}) < \numdirec < \frac{\dimn}{(\log 6)^2}$ we have 
with probability at least $1 - e^{-c_4^{\prime}\numdirec} - e^{-\sqrt{\numdirec\dimn}}$ that 
\begin{equation} \label{eq:grad_est_bd_gen}
\norm{\est{\grad} f(\vecx) - \grad f(\vecx)}_2 \leq C_1 \max\set{\norm{\taynoisvec}_2, \sqrt{\log \dimn}\norm{\taynoisvec}_{\infty}}. 
\end{equation}
Recall that $\taynoisvec = [\taynoissca_1 \dots \taynoissca_{\numdirec}]$
where $\taynoissca_j = \frac{\thirdtayrem_3(\zeta_{j}) - \thirdtayrem_3(\zeta^{\prime}_{j})}{2\gradstep}$, 
for some $\zeta_j,\zeta^{\prime}_j \in \matR^{\dimn}$.
Here $\thirdtayrem_3(\zeta)$ denotes the third order Taylor remainder terms of $f$. By taking the structure of $f$ into account, 
we can uniformly bound 
$\abs{\thirdtayrem_3(\zeta_j)}$ as follows (so the same bound holds for $\abs{\thirdtayrem_3(\zeta^{\prime}_j)}$).
Let us define $\numdegree := \abs{\set{q \in \bivsuppvar: \degree(q) > 1}}$, 
to be the number of variables in $\bivsuppvar$, with degree greater than one. 
\begin{align}
\abs{\thirdtayrem_3(\zeta_j)} &= \frac{\gradstep^3}{6} |\sum_{p \in \univsupp} \partial_p^3 \phi_p(\zeta_{j,p}) v_p^3  + 
\sum_{\lpair \in \bivsupp} ( \partial_l^3 \phi_{\lpair}(\zeta_{j,l}, \zeta_{j,{l^{\prime}}}) v_l^3 + \partial_{l^{\prime}}^3 
\phi_{\lpair}(\zeta_{j,l}, \zeta_{j,{l^{\prime}}}) v_{l^{\prime}}^3) \nonumber \\
&+ \sum_{\lpair \in \bivsupp} (3\partial_l \partial_{l^{\prime}}^2 \phi_{\lpair}(\zeta_{j,l}, \zeta_{j,{l^{\prime}}}) v_l v_{l^{\prime}}^2 + 
3\partial_l^2 \partial_{l^{\prime}} \phi_{\lpair}(\zeta_{j,l}, \zeta_{j,{l^{\prime}}}) v_l^2 v_{l^{\prime}}) \nonumber\\ 
&+ \sum_{q \in \bivsuppvar: \degree(q)>1} \partial_q^3 \phi_q(\zeta_{j,q}) v_q^3| \\
&\leq \frac{\gradstep^3}{6} \left(\frac{\univsparsity \smconst_3}{\numdirec^{3/2}} + 
\frac{2\bivsparsity\smconst_3}{\numdirec^{3/2}} + \frac{\alpha\smconst_3}{\numdirec^{3/2}} + \frac{6\bivsparsity\smconst_3}{\numdirec^{3/2}} \right) \\
&= \frac{\gradstep^3}{6}\frac{(\univsparsity + \alpha + 8\bivsparsity)\smconst_3}{\numdirec^{3/2}}.  \label{eq:temp_bd_1}
\end{align}
Using the fact $2\bivsparsity = \sum_{l \in \bivsuppvar: \degree(l) > 1} \degree(l) + (\abs{\bivsuppvar} - \numdegree)$, we can observe that 
$2\bivsparsity \leq \maxdegree\numdegree + (\abs{\bivsuppvar} - \numdegree) = \abs{\bivsuppvar} + (\maxdegree-1)\numdegree$. Plugging this in 
\eqref{eq:temp_bd_1}, and using the fact $\alpha \leq \totsparsity$ (since we do not assume $\alpha$ to be known), we obtain
\begin{align}
\abs{\thirdtayrem_3(\zeta_j)} &\leq \frac{\gradstep^3}{6}\frac{(\univsparsity + \numdegree + 4\abs{\bivsuppvar} + 4(\maxdegree-1)\numdegree)\smconst_3}{\numdirec^{3/2}} \\ 
&\leq \frac{\gradstep^3(4\totsparsity + (4\maxdegree-3)\numdegree)\smconst_3}{6\numdirec^{3/2}} 
\leq \frac{\gradstep^3((4\maxdegree+1)\totsparsity)\smconst_3}{6\numdirec^{3/2}}. 
\end{align}
This in turn implies that $\norm{\taynoisvec}_{\infty} \leq \frac{\gradstep^2((4\maxdegree+1)\totsparsity)\smconst_3}{6\numdirec^{3/2}}$. 
Using the fact $\norm{\taynoisvec}_{2} \leq \sqrt{\numdirec}\norm{\taynoisvec}_{\infty}$, we thus obtain for the stated choice of $\numdirec$ (cf. Remark \ref{rem:l1min_samp_bd}) that 
\begin{equation} \label{eq:grad_est_over_bd}
\norm{\est{\grad} f(\vecx) - \grad f(\vecx)}_2 \leq \frac{C_1\gradstep^2((4\maxdegree+1)\totsparsity)\smconst_3}{6\numdirec}, 
\quad \forall \vecx \in [-(1+r),1+r]^{\dimn}.
\end{equation}
Recall that $[-(1+r),1+r]^{\dimn}, r > 0$, denotes the enlargement around $[-1,1]^{\dimn}$, in which the smoothness 
properties of $\phi_p, \phi_{\lpair}$ are defined in Section \ref{sec:problem}. 
Since $\est{\grad} f(\vecx) = \grad f(\vecx) + \vecw(\vecx)$, therefore $\norm{\vecw(\vecx)}_{\infty} \leq \norm{\est{\grad} f(\vecx) - \grad f(\vecx)}_2$. 
Using the definition of $\hessestnoisb \in \matR^{\numdirecp}$ from \eqref{eq:hessrow_est_linfin}, we then have that 
$\norm{\hessestnoisb}_{\infty} \leq \frac{C_1\gradstep^2((4\maxdegree+1)\totsparsity)\smconst_3}{3\numdirec\hessstep}$.
\paragraph{Bounding the $\hessestnoisa$ term.}
We will bound $\norm{\hessestnoisa}_{\infty}$. To this end, we see from \eqref{eq:hessrow_est_linfin} that it 
suffices to uniformly bound $\abs{{\vecvp}^{T} \hess \partial_q f(\zeta) \vecvp}$, 
over all: $q \in \univsupp \cup \bivsuppvar$, $\vecvp \in \calVp$, $\zeta \in [-(1+r),(1+r)]^{\dimn}$. Note that
\begin{equation}
{\vecvp}^T \hess \partial_q f(\zeta) \vecvp = \sum_{l=1}^{\dimn} {v_l^{\prime}}^2 (\hess \partial_q f (\zeta))_{l,l} + 
\sum_{i \neq j = 1}^{\dimn} v_i^{\prime}v_j^{\prime} (\hess \partial_q f (\zeta))_{i,j}.
\end{equation}
We have the following three cases, depending on the type of $q$.
\begin{enumerate}
\item $\mathbf{q \in \univsupp.}$
\begin{equation}
{\vecvp}^T \hess \partial_q f(\zeta) \vecvp = {v^{\prime}_{q}}^2 \partial_q^3 \phi_q(\zeta_q) \Rightarrow 
\abs{{\vecvp}^T \hess \partial_q f(\zeta) \vecvp} \leq \frac{\smconst_3}{\numdirecp}.
\end{equation}

\item $\mathbf{\qpair \in \bivsupp}$, $\mathbf{\degree(q) = 1.}$
\begin{align}
{\vecvp}^T \hess \partial_q f(\zeta) \vecvp &= {v^{\prime}_{q}}^2 \partial_q^3 \phi_{\qpair}(\zeta_q,\zeta_{q^{\prime}}) + 
{v^{\prime}_{q^{\prime}}}^2 \partial_{q^{\prime}}^2 \partial_q \phi_{\qpair}(\zeta_q,\zeta_{q^{\prime}}) \\ 
&+ 
2 v^{\prime}_{q} v^{\prime}_{q^{\prime}} \partial_{q^{\prime}} \partial_q^2 \phi_{\qpair}(\zeta_q,\zeta_{q^{\prime}}), \\
\Rightarrow \abs{{\vecvp}^T \hess \partial_q f(\zeta) \vecvp} &\leq \frac{4\smconst_3}{\numdirecp}.
\end{align}

\item $\mathbf{q \in \bivsuppvar}$, $\mathbf{\degree(q) > 1.}$
\begin{align}
{\vecvp}^T \hess \partial_q f(\zeta) \vecvp &= {v^{\prime}_{q}}^2(\partial_q^3 \phi_q(\zeta_q) + 
\sum_{\qpair \in \bivsupp} \partial_q^3 \phi_{\qpair}(\zeta_{q},\zeta_{q^{\prime}}) \nonumber \\ &+ 
\sum_{\qpairi \in \bivsupp} \partial_q^3 \phi_{\qpairi}(\zeta_{q^{\prime}},\zeta_{q})) 
+ \sum_{\qpair \in \bivsupp} {v^{\prime}_{q^{\prime}}}^2 \partial_{q^{\prime}}^2 \partial_q \phi_{\qpair}(\zeta_{q},\zeta_{q^{\prime}}) \nonumber \\ &+ 
\sum_{\qpairi \in \bivsupp} {v^{\prime}_{q^{\prime}}}^2 \partial_{q^{\prime}}^2 \partial_q \phi_{\qpairi}(\zeta_{q^{\prime}},\zeta_{q}) 
+ 2\sum_{\qpair \in \bivsupp} v^{\prime}_{q} v^{\prime}_{q^{\prime}} \partial_{q^{\prime}} \partial_q^2 \phi_{\qpair}(\zeta_{q},\zeta_{q^{\prime}}) \nonumber \\ &+ 
2\sum_{\qpairi \in \bivsupp} v^{\prime}_{q} v^{\prime}_{q^{\prime}} \partial_{q^{\prime}} \partial_q^2 \phi_{\qpairi}(\zeta_{q^{\prime}},\zeta_{q}), \\
\Rightarrow \abs{{\vecvp}^T \hess \partial_q f(\zeta) \vecvp} &\leq \frac{1}{\numdirecp}((\maxdegree+1)\smconst_3 + \maxdegree\smconst_3 + 2\maxdegree\smconst_3)
= \frac{(4\maxdegree+1)\smconst_3}{\numdirecp}.
\end{align}
\end{enumerate}
We can now uniformly bound $\norm{\hessestnoisa}_{\infty}$ as follows.
\begin{equation}
\norm{\hessestnoisa}_{\infty} := \max_{j=1,\dots,\numdirecp} \frac{\hessstep}{2} \abs{{{\vecvp_j}^T \hess \partial_q f(\zeta_j) \vecvp_j}} 
\leq \frac{\hessstep(4\maxdegree+1)\smconst_3}{2\numdirecp}. 
\end{equation}
\paragraph{Estimating $\bivsupp$.}
We now proceed towards estimating $\bivsupp$. To this end, we estimate $\grad \partial_q f(\vecx)$ for each $q=1,\dots,\dimn$ and $\vecx \in \baseset$.
Since $\grad \partial_q f(\vecx)$ is at most $(\maxdegree+1)$-sparse, therefore Theorem \ref{thm:sparse_recon_bound}, \eqref{eq:sparse_recon_err}, 
immediately yield the following.
$\exists C_2, c_5^{\prime}> 0, c_2^{\prime} \geq 1$ such that for 
$c_2^{\prime} \maxdegree \log(\frac{\dimn}{\maxdegree}) < \numdirecp < \frac{\dimn}{(\log 6)^2}$ we have 
with probability at least $1 - e^{-c_5^{\prime}\numdirecp} - e^{-\sqrt{\numdirecp\dimn}}$ that 
\begin{equation} \label{eq:hessrow_est_bd_gen}
\norm{\est{\grad} \partial_q f(\vecx) - \grad \partial_q f(\vecx)}_2 \leq C_2 \max\set{\norm{\hessestnoisa+\hessestnoisb}_2, \sqrt{\log \dimn}\norm{\hessestnoisa+\hessestnoisb}_{\infty}}. 
\end{equation} 
Since $\norm{\hessestnoisa+\hessestnoisb}_{\infty} \leq \norm{\hessestnoisa}_{\infty} + \norm{\hessestnoisb}_{\infty}$, 
therefore using the bounds on $\norm{\hessestnoisa}_{\infty},\norm{\hessestnoisb}_{\infty}$ and noting that 
$\norm{\hessestnoisa+\hessestnoisb}_2 \leq \sqrt{\numdirecp}\norm{\hessestnoisa+\hessestnoisb}_{\infty}$, we obtain for the stated choice of $\numdirecp$ (cf. Remark \ref{rem:l1min_samp_bd}) that
\begin{equation} \label{eq:hessrow_est_bd_gen_1}
\norm{\est{\grad} \partial_q f(\vecx) - \grad \partial_q f(\vecx)}_2 \leq \underbrace{C_2\left(\frac{\hessstep(4\maxdegree+1)\smconst_3}{2\sqrt{\numdirecp}} 
+ \frac{C_1\sqrt{\numdirecp}\gradstep^2((4\maxdegree+1)\totsparsity)\smconst_3}{3\numdirec\hessstep} \right)}_{\hesssamperr}.
\end{equation}
for $q=1,\dots,\dimn$, and $\forall \vecx \in [-1,1]^{\dimn}$. We next note that \eqref{eq:hessrow_est_bd_gen_1} trivially leads to the bound
\begin{equation}
\est{\partial_q\partial_{q^{\prime}}} f(\vecx) \in [\partial_q\partial_{q^{\prime}} f(\vecx) -\hesssamperr, \partial_q\partial_{q^{\prime}} f(\vecx) + \hesssamperr]; \quad
q,q^{\prime} = 1,\dots,\dimn. 
\end{equation}
Now if $q \notin \bivsuppvar$ then clearly $\est{\partial_q\partial_{q^{\prime}}} f(\vecx) \in [-\hesssamperr, \hesssamperr]$; $\forall \vecx \in [-1,1]^{\dimn}, q \neq q^{\prime}$.
On the other hand, if $\qpair \in \bivsupp$ then 
\begin{equation}
\est{\partial_q\partial_{q^{\prime}}} f(\vecx) \in [\partial_q\partial_{q^{\prime}} \phi_{\qpair}(x_q,x_{q^{\prime}}) -\hesssamperr, \partial_q\partial_{q^{\prime}} \phi_{\qpair}(x_q,x_{q^{\prime}}) + \hesssamperr]. 
\end{equation}
If furthermore $\numcen \geq \critintmeas_2^{-1}$, then due to the construction of $\baseset$, $\exists \vecx \in \baseset$ so that 
$\abs{\est{\partial_q\partial_{q^{\prime}}} f(\vecx)} \geq \idenconst_2 - \hesssamperr$. Hence if $\hesssamperr < \idenconst_2/2$ holds, 
the we would have $\abs{\est{\partial_q\partial_{q^{\prime}}} f(\vecx)} > \idenconst_2/2$, leading to the identification of $\qpair$.
Since this is true for each $\qpair \in \bivsupp$, hence it follows that $\est{\bivsupp} = \bivsupp$. Now, $\hesssamperr < \idenconst_2/2$ is equivalent to 
\begin{align}
\underbrace{\frac{(4\maxdegree+1)\smconst_3}{2\sqrt{\numdirecp}}}_{a} \hessstep
+ \underbrace{\left(\frac{C_1\sqrt{\numdirecp}((4\maxdegree+1)\totsparsity)\smconst_3}{3\numdirec}\right)}_{b}\frac{\gradstep^2}{\hessstep} &< \frac{\idenconst_2}{2C_2} \\
\Leftrightarrow a \hessstep^2 - \frac{\idenconst_2}{2C_2}\hessstep + b\gradstep^2 &< 0 \\
\Leftrightarrow \hessstep \in ((\idenconst_2/(4aC_2)) - \sqrt{(\idenconst_2/(4aC_2))^2 - (b\gradstep^2/a)} , (\idenconst_2/(4aC_2)) &+ \sqrt{(\idenconst_2/(4aC_2))^2 - (b\gradstep^2/a)} ). 
\label{eq:hessstep_bd_gen_overlap}
\end{align}
Lastly, we see that the bounds in \eqref{eq:hessstep_bd_gen_overlap} are valid if:
\begin{equation} \label{eq:gradstep_bd_gen_over}
\gradstep^2 < \frac{\idenconst_2^2}{16 a b C_2^2} = \frac{3 \idenconst_2^2 \numdirec}{8C_1 C_2^2 \smconst_3^2(4\maxdegree+1)((4\maxdegree+1)\totsparsity)}.
\end{equation}
%

\paragraph{Estimating $\univsupp$.} 
With $\calP := [\dimn] \setminus \est{\bivsuppvar}$, we have via Taylor's expansion of $f$ at $j=1,\dots,\numdirecpp$:
\begin{equation} \label{eq:taylor_exp_f_2}
\frac{f((\vecx + \gradstepp\vecvpp_j)_{\calP}) - f((\vecx - \gradstepp\vecvpp_j)_{\calP})}{2\gradstepp} = \dotprod{(\vecvpp_j)_{\calP}}{(\grad f((\vecx)_{\calP}))_{\calP}} 
+ \underbrace{\frac{\thirdtayrem_3((\zeta_j)_{\calP}) - \thirdtayrem_3((\zeta_j^{\prime})_{\calP})}{2\gradstepp}}_{\taynoissca_j}.
\end{equation}
\eqref{eq:taylor_exp_f_2} corresponds to linear measurements of the $(\totsparsity-\abs{\est{\bivsuppvar}})$ sparse vector: $(\grad f(\vecx_{\calP}))_{\calP}$.
Note that we effectively perform $\ell_1$ minimization over $\matR^{\abs{\calP}}$. 
Therefore for any $\vecx \in \matR^{\dimn}$ we immediately have from Theorem \ref{thm:sparse_recon_bound}, \eqref{eq:sparse_recon_err}, the following. 
$\exists C_3, c_6^{\prime}> 0, c_3^{\prime} \geq 1$ such that for 
$c_3^{\prime} (\totsparsity-\abs{\est{\bivsuppvar}}) \log(\frac{\abs{\calP}}{\totsparsity-\abs{\est{\bivsuppvar}}}) < \numdirecpp < \frac{\abs{\calP}}{(\log 6)^2}$, we have 
with probability at least $1 - e^{-c_6^{\prime}\numdirecpp} - e^{-\sqrt{\numdirecpp\abs{\calP}}}$ that 
\begin{equation} \label{eq:grad_est_bd_gen_1}
\norm{(\est{\grad} f((\vecx)_{\calP}))_{\calP} - (\grad f((\vecx)_{\calP}))_{\calP}}_2 \leq C_3 \max\set{\norm{\taynoisvec}_2, \sqrt{\log \abs{\calP}} \norm{\taynoisvec}_{\infty}},
\end{equation}
where $\vecn = [n_1 \cdots n_{\numdirecpp}]$. We now uniformly bound $\thirdtayrem_3((\zeta_j)_{\calP})$ for all $j=1,\dots,\numdirecpp$ and $\zeta_j \in [-(1+r),1+r]^{\dimn}$ as follows.
\begin{align}
\thirdtayrem_3((\zeta_j)_{\calP}) = \frac{{\gradstepp}^3}{6}\sum_{p \in \univsupp \cap \calP} \partial_p^3 \phi_p(\zeta_{j,p}){\vpp_{j,p}}^3 \quad 
\Rightarrow \abs{\thirdtayrem_3((\zeta_j)_{\calP})} \leq \frac{(\totsparsity-\abs{\est{\bivsuppvar}}) {\gradstepp}^3\smconst_3}{6 \numdirecpp^{3/2}}.
\end{align}
This in turn implies that $\norm{\taynoisvec}_{\infty} \leq \frac{(\totsparsity-\abs{\est{\bivsuppvar}}) {\gradstepp}^2\smconst_3}{6 \numdirecpp^{3/2}}$ and 
$\norm{\taynoisvec}_2 \leq \sqrt{\numdirecpp}\norm{\taynoisvec}_{\infty} \leq \frac{(\totsparsity-\abs{\est{\bivsuppvar}}) {\gradstepp}^2\smconst_3}{6 \numdirecpp}$. 
Plugging these bounds in \eqref{eq:grad_est_bd_gen_1}, we obtain for the stated choice of $\numdirecpp$ (cf. Remark \ref{rem:l1min_samp_bd}) that 
\begin{equation} \label{eq:grad_est_bd_gen_2}
\norm{(\est{\grad} f((\vecx)_{\calP}))_{\calP} - (\grad f((\vecx)_{\calP}))_{\calP}}_2 \leq 
\underbrace{\frac{C_3 (\totsparsity-\abs{\est{\bivsuppvar}}) {\gradstepp}^2\smconst_3}{6 \numdirecpp}}_{\derivsamperrpp} ; \quad \vecx \in [-1,1]^{\dimn}. 
\end{equation}
Finally, using the same arguments as before, we have that $\derivsamperrpp < \idenconst_1/2$ or equivalently 
${\gradstepp}^2 < \frac{3\numdirecpp \idenconst_1}{C_3 (\totsparsity-\abs{\est{\bivsuppvar}}) \smconst_3}$ is sufficient to recover $\univsupp$. This 
completes the proof.

\subsection{Proof of Theorem \ref{thm:gen_overlap_arbnois}} \label{subsec:proof_thm_genover_arbnoise}

We prove a more detailed version of Theorem \ref{thm:gen_overlap_arbnois}, stated below.

\begin{theorem} \label{thm:gen_overlap_arbnois_det}
Assuming notation in Theorem \ref{thm:gen_overlap}, let $\numcen, \numcenpair, \numdirec, \numdirecp, \numdirecpp$ be as defined in 
Theorem \ref{thm:gen_overlap}. Say $\exnoisemag < \exnoisemag_1 = \frac{\idenconst_2^{3}}{192\sqrt{3} C_1 C_2^3 \sqrt{a^3 b \numdirecp \numdirec}}$. 
Denoting $\theta_1 = \cos^{-1}(-\exnoisemag / \exnoisemag_1)$, $b^{\prime} = 2C_1\sqrt{\numdirec\numdirecp}$, 
we have for  
%
$\gradstep \in (\sqrt{4{a^{\prime}}^2 a/(3b)}\cos(\theta_1/3 - 2\pi/3) , \sqrt{4{a^{\prime}}^2 a/(3b)}\cos(\theta_1/3))$ and  
$\hessstep \in (a^{\prime} - \sqrt{{a^{\prime}}^2 - \left((b\gradstep^2 + b^{\prime} \exnoisemag)/a\right)}$, 
$a^{\prime} + \sqrt{{a^{\prime}}^2 - \left((b\gradstep^2 + b^{\prime} \exnoisemag)/a\right)})$  
%
that $\hesssamperr = C_2 \left(a\hessstep+ \frac{b \gradstep^2}{\hessstep} + \frac{b^{\prime}\exnoisemag}{\gradstep\hessstep}\right)$
implies $\est{\bivsupp} = \bivsupp$ with high probability. Given $\est{\bivsupp} = \bivsupp$, 
denote $a_1 = \frac{(\totsparsity-\abs{\est{\bivsuppvar}}) \smconst_3}{6\numdirecpp}$, $b_1 = \sqrt{\numdirecpp}$ 
and say $\exnoisemag < \exnoisemag_2 = \frac{\idenconst_1^{3/2}}{3\sqrt{6 a_1 C_3^3 b_1^2}}$. 
For $\theta_2 = \cos^{-1}(-\exnoisemag / \exnoisemag_2)$, 
let $\gradstepp \in (2\sqrt{\idenconst_1/(6 a_1 C_3)} \cos(\theta_2/3 - 2\pi/3), 2\sqrt{\idenconst_1/(6 a_1 C_3)} \cos(\theta_2/3))$. 
Then $\derivsamperrpp = C_3(a_1 {\gradstepp}^2 + \frac{b_1\exnoisemag}{\gradstepp})$
implies $\est{\univsupp} = \univsupp$ with high probability.
\end{theorem}

\begin{proof}

We begin by establishing the conditions pertaining to the estimation of $\bivsupp$. Then  
we prove the conditions for estimation of $\univsupp$.
\paragraph{Estimation of $\bivsupp$.} We first note that the linear system \eqref{eq:cs_form} 
now has the form: $\vecy = \matV\grad f(\vecx) + \taynoisvec + \exnoisevec$ where 
$\exnoise_{j} = (\exnoisep_{j,1} - \exnoisep_{j,2})/(2\gradstep)$ represents the external noise
component, for $j=1,\dots,\numdirec$. Observe that $\norm{\exnoisevec}_{\infty} \leq \exnoisemag/\gradstep$.
Using the bounds on $\norm{\taynoisvec}_{\infty}, \norm{\taynoisvec}_{2}$ from Section \ref{subsec:proof_thm_genover}, 
we then observe that \eqref{eq:grad_est_over_bd} changes to:
\begin{equation} \label{eq:grad_est_over_bd_arbnois}
\norm{\est{\grad} f(\vecx) - \grad f(\vecx)}_2 \leq C_1\left(\frac{\gradstep^2((4\maxdegree+1)\totsparsity)\smconst_3}{6\numdirec} + \frac{\exnoisemag\sqrt{\numdirec}}{\gradstep}\right), 
\quad \forall \vecx \in [-(1+r),1+r]^{\dimn}.
\end{equation}
As a result, we then have that 
\begin{equation}
\norm{\hessestnoisb}_{\infty} \leq C_1\left(\frac{\gradstep^2((4\maxdegree+1)\totsparsity)\smconst_3}{3\numdirec\hessstep} + \frac{2\exnoisemag\sqrt{\numdirec}}{\gradstep\hessstep}\right).
\end{equation}
Now note that the bound on $\norm{\hessestnoisa}_{\infty}$ is unchanged from Section \ref{subsec:proof_thm_genover} i.e., 
$\norm{\hessestnoisa}_{\infty} \leq \frac{\hessstep(4\maxdegree+1)\smconst_3}{2\numdirecp}$. As a consequence, we see that 
\eqref{eq:hessrow_est_bd_gen_1} changes to:
\begin{equation} \label{eq:hessrow_est_bd_gen_arbnois}
\norm{\est{\grad} \partial_q f(\vecx) - \grad \partial_q f(\vecx)}_2 \leq \underbrace{C_2\left(\frac{\hessstep(4\maxdegree+1)\smconst_3}{2\sqrt{\numdirecp}} 
+ C_1\frac{\sqrt{\numdirecp}\gradstep^2((4\maxdegree+1)\totsparsity)\smconst_3}{3\numdirec\hessstep} + \frac{2C_1\exnoisemag\sqrt{\numdirec\numdirecp}}{\gradstep\hessstep}\right)}_{\hesssamperr}.
\end{equation}
With $a$ and $b$ as stated in the Theorem,  we then see that $\hesssamperr < \idenconst_2/2$ is equivalent to
\begin{equation}
a\hessstep^2 - \frac{\idenconst_2}{2C_2}\hessstep + \left(b\gradstep^2 + \frac{2C_1\exnoisemag\sqrt{\numdirec\numdirecp}}{\gradstep}\right) < 0.
\end{equation}
which in turn is equivalent to 
\begin{equation}
\hessstep \in \left(\frac{\idenconst_2}{4aC_2} - \sqrt{\left(\frac{\idenconst_2}{4aC_2}\right)^2 - 
\left(\frac{b\gradstep^3 + 2C_1\exnoisemag\sqrt{\numdirec\numdirecp}}{a\gradstep}\right)} , 
\frac{\idenconst_2}{4aC_2} + \sqrt{\left(\frac{\idenconst_2}{4aC_2}\right)^2 - \left(\frac{b\gradstep^3 + 2C_1\exnoisemag\sqrt{\numdirec\numdirecp}}{a\gradstep}\right)}\right).
\end{equation}
For the above bound to be valid, we require 
\begin{align}
\frac{b\gradstep^2}{a} + \frac{2C_1\exnoisemag\sqrt{\numdirec\numdirecp}}{a\gradstep} &< \frac{\idenconst_2^2}{16a^2C_2^2}, \\
\Leftrightarrow \gradstep^3 - \frac{\idenconst_2^2}{16abC_2^2}\gradstep + \frac{2C_1\exnoisemag\sqrt{\numdirec\numdirecp}}{b} &< 0 \label{eq:cub_gradstep}
\end{align}
to hold. \eqref{eq:cub_gradstep} is a cubic inequality. A cubic equation of the form: $y^3 + py + q = 0$, has $3$ distinct real roots 
if its discriminant $\frac{p^3}{27} + \frac{q^2}{4} < 0$. Note that for this to be possible, $p$ must be negative, which is the case in \eqref{eq:cub_gradstep}. 
Applying this to \eqref{eq:cub_gradstep} leads to the condition:
$\exnoisemag < \frac{\idenconst_2^{3}}{192\sqrt{3} C_1 C_2^3 \sqrt{a^3 b \numdirecp \numdirec}} = \exnoisemag_1$. 
Furthermore, the $3$ distinct real roots are given by:
\begin{equation} 
y_1 = 2\sqrt{-p/3}\cos(\theta/3), \ y_2 = -2\sqrt{-p/3}\cos(\theta/3 + \pi/3), \ y_3 = -2\sqrt{-p/3}\cos(\theta/3 - \pi/3) \label{eq:cub_roots}
\end{equation}
where $\theta = \cos^{-1}\left(\frac{-q/2}{\sqrt{-p^3/27}}\right)$. Applying this to \eqref{eq:cub_gradstep} 
then leads to $\theta_1 = \cos^{-1}(-\exnoisemag/\exnoisemag_1)$.
For $0 < \exnoisemag < \exnoisemag_1$ we have $\pi/2 < \theta_1 < \pi$ which implies $0 < y_2 < y_1$ and $y_3 < 0$. In particular 
if $q > 0$, then one can verify that $y^3 + py + q < 0$ holds if $y \in (y_2,y_1)$. 
Applying this to \eqref{eq:cub_gradstep}, we consequently obtain: 
\begin{equation}
\gradstep \in \left(\sqrt{\frac{\idenconst_2^2}{12 a b C_2^2}}\cos(\theta_1/3 - 2\pi/3) , \sqrt{\frac{\idenconst_2^2}{12 a b C_2^2}}\cos(\theta_1/3)\right). 
\end{equation}
%

\paragraph{Estimation of $\univsupp$.} We now prove the conditions for estimation of $\univsupp$. 
First note that \eqref{eq:taylor_exp_f_2} now changes to:
\begin{equation} \label{eq:taylor_exp_f_3}
\frac{f((\vecx + \gradstepp\vecvpp_j)_{\calP}) - f((\vecx - \gradstepp\vecvpp_j)_{\calP})}{2\gradstepp} = \dotprod{(\vecvpp_j)_{\calP}}{(\grad f((\vecx)_{\calP}))_{\calP}} 
+ \underbrace{\frac{\thirdtayrem_3((\zeta_j)_{\calP}) - \thirdtayrem_3((\zeta_j^{\prime})_{\calP})}{2\gradstepp}}_{\taynoissca_j} + \underbrace{\frac{\exnoisep_{j,1} - \exnoisep_{j,2}}{2\gradstepp}}_{\exnoise_j}, 
\end{equation}
for $j=1,\dots,\numdirecpp$. Denoting $\exnoisevec = [\exnoise_1 \cdots \exnoise_{\numdirecpp}]$, we have $\norm{\exnoisevec}_{\infty} \leq \exnoisemag/\gradstepp$. 
As the bounds on $\norm{\taynoisvec}_{2}, \norm{\taynoisvec}_{\infty}$ are unchanged, therefore \eqref{eq:grad_est_bd_gen_3} now changes to:
\begin{equation} \label{eq:grad_est_bd_gen_3}
\norm{(\est{\grad} f((\vecx)_{\calP}))_{\calP} - (\grad f((\vecx)_{\calP}))_{\calP}}_2 \leq 
\underbrace{C_3 \left(\frac{(\totsparsity-\abs{\est{\bivsuppvar}}) {\gradstepp}^2\smconst_3}{6 \numdirecpp} + \frac{\exnoisemag\sqrt{\numdirecpp}}{\gradstepp}\right)}_{\derivsamperrpp}  ; \quad \vecx \in [-1,1]^{\dimn}. 
\end{equation}
Denoting $a_1 = \frac{(\totsparsity-\abs{\est{\bivsuppvar}}) \smconst_3}{6\numdirecpp}$, $b_1 = \sqrt{\numdirecpp}$, we then see from 
\eqref{eq:grad_est_bd_gen_3} that the condition $\derivsamperrpp < \idenconst_1/2$ is equivalent to 
\begin{equation} \label{eq:s1_noise_cub}
{\gradstepp}^3 - \frac{\idenconst_1}{2 a_1 C_3}\gradstepp + \frac{b_1 \exnoisemag}{a_1} < 0.
\end{equation}
As discussed earlier for estimation of $\bivsupp$, the cubic equation corresponding to \eqref{eq:s1_noise_cub} has 
$3$ distinct real roots if its discriminant is negative. This then 
leads to the condition $\exnoisemag < \frac{\idenconst_1^{3/2}}{3\sqrt{6 a_1 C_3^3 b_1^2}} = \exnoisemag_2$. 
Then by using the expressions for the roots of the cubic from \eqref{eq:cub_roots}, one can verify 
that \eqref{eq:s1_noise_cub} holds if 
\begin{equation}
\gradstepp \in (2\sqrt{\idenconst_1/(6 a_1 C_3)} \cos(\theta_2/3 - 2\pi/3), 2\sqrt{\idenconst_1/(6 a_1 C_3)} \cos(\theta_2/3))
\end{equation}
with $\theta_2 = \cos^{-1}(-\exnoisemag/\exnoisemag_2)$. This completes the proof. 

\end{proof}

\subsection{Proof of Theorem \ref{thm:gen_overlap_gaussnois}} \label{subsec:proof_thm_genover_gauss}
We first derive conditions for estimating $\bivsupp$, and then for $\univsupp$.
\paragraph{Estimating $\bivsupp$.} Upon resampling $N_1$ times and averaging, we have for the
noise vector $\vecz \in \matR^{\numdirec}$ where  
\begin{equation}
\vecz = \left[\frac{(\exnoisep_{1,1} - \exnoisep_{1,2})}{2\gradstep} \cdots \frac{(\exnoisep_{\numdirec,1} - \exnoisep_{\numdirec,2})}{2\gradstep} \right], 
\end{equation}
that $\exnoisep_{j,1}, \exnoisep_{j,2} \sim \calN(0,\sigma^2/N_1)$ are i.i.d. Note that it is in fact sufficient to guarantee that 
$\abs{\exnoisep_{j,1} - \exnoisep_{j,2}} < 2\exnoisemag$ holds $\forall j=1,\dots,\numdirec$, and across all points where 
$\grad f$ is estimated. Indeed, we can then simply use the proof in Section \ref{subsec:proof_thm_genover_arbnoise}, 
for the setting of arbitrary bounded noise. 
To this end, note that $\exnoisep_{j,1}-\exnoisep_{j,2} \sim \calN(0,\frac{2\sigma^2}{N_1})$. 
It can be shown for $X \sim \calN(0,1)$ that:
\begin{equation}
\prob(\abs{X} > t) \leq \frac{2 e^{-t^2/2}}{t}, \quad \forall t > 0.
\end{equation}
Since $\exnoisep_{j,1}-\exnoisep_{j,2} = \sigma\sqrt{\frac{2}{N_1}} X$ therefore for any $\exnoisemag > 0$ we have that:
\begin{align}
\prob(\abs{\exnoisep_{j,1} - \exnoisep_{j,2}} > 2\exnoisemag) &= \prob\left(\abs{X} > \frac{2\exnoisemag}{\sigma}\sqrt{\frac{N_1}{2}}\right) \\
&\leq \frac{\sigma}{\exnoisemag}\sqrt{\frac{2}{N_1}} \exp\left(-\frac{\exnoisemag^2 N_1}{\sigma^2}\right) \\
&\leq \frac{\sqrt{2}\sigma}{\exnoisemag} \exp\left(-\frac{\exnoisemag^2 N_1}{\sigma^2}\right).
\end{align}
Now to estimate $\grad f(\vecx)$ we have $\numdirec$ many ``difference'' terms: $\exnoisep_{j,1} - \exnoisep_{j,2}$. We additionally estimate 
$\numdirecp$ many gradients at each $\vecx$ implying a total of $\numdirec(\numdirecp + 1)$ difference terms. As this is done for each 
$\vecx \in \baseset$, therefore we have a total of $\numdirec(\numdirecp + 1)(2\numcen+1)^2\abs{\twohashfam}$ many difference terms. 
Taking a union bound over all of them, we have for any $p_1 \in (0,1), \exnoisemag > 0$ 
that the choice $N_1 > \frac{\sigma^2}{\exnoisemag^2} \log (\frac{\sqrt{2} \sigma}{\exnoisemag p_1}\numdirec(\numdirecp+1)(2\numcen+1)^2\abs{\twohashfam})$ 
implies that the magnitudes of all difference terms are bounded by $2\exnoisemag$, with probability at least 
$1-p_1$. Thereafter, we can simply follow the proof in Section \ref{subsec:proof_thm_genover_arbnoise}, 
for estimating $\bivsupp$ in the presence of arbitrary bounded noise.

\paragraph{Estimating $\univsupp$.} In this case, we resample each query $N_2$ times and average -- therefore the variance of the noise 
terms gets scaled by $N_2$. We now have $\abs{\baseset_{\text{diag}}} \numdirecpp = (2\numcenpair+1) \numdirecpp$ many ``difference'' terms corresponding to Gaussian noise. Therefore, taking a union bound over all of them, we have for any $p_2 \in (0,1), \exnoisemagp > 0$ 
that the choice $N_2 > \frac{\sigma^2}{{\exnoisemagp}^2} \log(\frac{\sqrt{2} \sigma (2\numcenpair+1)\numdirecpp}{\exnoisemagp p_2})$ 
implies that the magnitudes of all difference terms are bounded by $2\exnoisemagp$, with probability at least 
$1-p_2$. Thereafter, we can simply follow the proof in Section \ref{subsec:proof_thm_genover_arbnoise}, 
for estimating $\univsupp$ in the presence of arbitrary bounded noise. The only change there would be to replace $\exnoisemag$ by $\exnoisemagp$.  
\section{Learning individual components of model} \label{sec:est_comp}
Recall from \eqref{eq:unique_mod_rep} the unique representation of the model:
\begin{equation} 
f(x_1,\dots,x_d) = c + \sum_{p \in \univsupp}\phi_{p} (x_p) + \sum_{\lpair \in \bivsupp} \phi_{\lpair} \xlpair + \sum_{q \in \bivsuppvar: \degree(q) > 1} \phi_{q} (x_q), 
\end{equation}
where $\univsupp \cap \bivsuppvar = \emptyset$. Having estimated the sets $\univsupp$ and $\bivsupp$, 
we now show how the individual univariate and bivariate functions in the model can be estimated. 
We will see this for the settings of noiseless, as well as noisy (arbitrary, bounded noise and stochastic noies) point queries.

\subsection{Noiseless queries}
In this scenario, we obtain the exact value $f(\vecx)$ at each query $\vecx \in \matR^{\dimn}$.
Let us first see how each $\phi_p$; $p \in \univsupp$ can be estimated.
For some $-1 = t_1 < t_2 < \dots < t_n = -1 $, consider the set
\begin{equation} \label{eq:univ_est_set}
\baseset_p := \left\{\vecx_i \in \matR^{\dimn}: (\vecx_i)_j =  \left\{
\begin{array}{rl}
t_i ; & j = p, \\
0 ; & j \neq p 
\end{array} \right\} ; 1 \leq i \leq n; 1 \leq j \leq \dimn\right\}; \quad p \in \univsupp.
\end{equation}
We obtain the samples $\set{f(\vecx_i)}_{i=1}^{n}$; $\vecx_i \in \baseset_p$. Here $f(\vecx_i) = \phi_p(t_i) + C$ with $C$ being a constant that 
depends on the other components in the model. Given the samples, one can then employ spline based ``quasi interpolant operators'' \cite{deBoor78}, 
to obtain an estimate $\phitil_p :[-1,1] \rightarrow \matR$, to $\phi_p + C$. 
Construction of such operators can be found for instance in \cite{deBoor78} (see also \cite{Gyorfi2002}). 
One can suitably choose the $t_i$'s and construct quasi interpolants that approximate any $C^m$ 
smooth univariate function with optimal $\Linfnorm[-1,1]$ error rate $O(n^{-m})$ \cite{deBoor78, Gyorfi2002}. 
Having obtained $\phitil_p$, we then define 
\begin{equation} \label{eq:univ_est}
\est{\phi}_p := \phitil_p -\expec_p[\phitil_p]; \quad p \in \univsupp, 
\end{equation}
to be the estimate of $\phi_p$. The bivariate components corresponding to each $\lpair \in \bivsupp$ 
can be estimated in a similar manner as above. To this end, for some 
strictly increasing sequences: $(-1 = \tp_1,\tp_2,\dots,\tp_{n_1} = 1)$, $(-1 = t_1,t_2,\dots,t_{n_1} = 1)$, 
consider the set 
\begin{equation} \label{eq:biv_est_set}
\baseset_{\lpair} := \left\{\vecx_{i,j} \in \matR^{\dimn}: (\vecx_{i,j})_q =  \left\{
\begin{array}{rl}
\tp_i ; & q = l, \\
t_j ; & q = l^{\prime}, \\
0 ; & q \neq l,l^{\prime} 
\end{array} \right\}  ; 1 \leq i,j \leq n_1; 1 \leq q \leq \dimn \right\}; \quad \lpair \in \bivsupp.
\end{equation}
We then obtain the samples $\set{f(\vecx_{i,j})}_{i,j=1}^{n_1}$; $\vecx_{i,j} \in \baseset_{\lpair}$ where
\begin{align} \label{eq:biv_part_fn_exp}
 f(\vecx_{i,j}) &= \phi_{\lpair}(\tp_i,t_j) + \sum_{\substack{l_1:(l,l_1) \in \bivsupp \\ l_1 \neq \lp}} \phi_{(l,l_1)} (\tp_i,0) + \sum_{\substack{l_1:(l_1,l) \in \bivsupp \\ l_1 \neq \lp}} 
\phi_{(l_1,l)} (0,\tp_i) \nonumber \\
&+ \sum_{\substack{\lp_1:(\lp,\lp_1) \in \bivsupp \\ \lp_1 \neq l}} \phi_{(\lp,\lp_1)} (t_j,0) + \sum_{\substack{\lp_1:(\lp_1,\lp) \in \bivsupp \\ \lp_1 \neq l}} \phi_{(\lp_1,\lp)} (0,t_j) + 
\phi_l(\tp_i) + \phi_{\lp}(t_j) + C, \\
&= g_{\lpair}(\tp_i,t_j) + C,
\end{align}
with $C$ being a constant. \eqref{eq:biv_part_fn_exp} is a general expression -- if for example $\degree(l) = 1$, then 
the terms $\phi_l,\phi_{(l,l_1)},\phi_{(l_1,l)}$ will be zero.
Given this, we can again obtain estimates $\phitil_{\lpair}:[-1,1]^2 \rightarrow \matR$ to $g_{\lpair} + C$, 
via spline based quasi interpolants.
Let us denote $n = n_1^2$ to be the total number of samples of $f$. For an appropriate choice of $(\tp_i,t_j)$'s, 
one can construct bivariate quasi interpolants that approximate any $C^m$ smooth bivariate function, with optimal 
$\Linfnorm[-1,1]^2$ error rate $O(n^{-m/2})$ \cite{deBoor78, Gyorfi2002}. 
Subsequently, we define the final estimates $\est{\phi}_{\lpair}$ to $\phi_{\lpair}$ as follows.
\begin{equation} \label{eq:biv_est}
\est{\phi}_{\lpair} :=  \left\{
\begin{array}{rl}
\phitil_{\lpair} - \expec_{\lpair}[\phitil_{\lpair}] ; & \degree(l), \degree(l^{\prime}) = 1, \\
\phitil_{\lpair} - \expec_{l}[\phitil_{\lpair}] ; &  \degree(l) = 1, \degree(l^{\prime}) > 1, \\
\phitil_{\lpair} - \expec_{l^{\prime}}[\phitil_{\lpair}] ; & \degree(l) > 1, \degree(l^{\prime}) = 1, \\
\phitil_{\lpair} - \expec_{l}[\phitil_{\lpair}] - \expec_{l^{\prime}}[\phitil_{\lpair}] + \expec_{\lpair}[\phitil_{\lpair}] ; & \degree(l) > 1, \degree(l^{\prime}) > 1.
\end{array} \right. 
\end{equation}
Lastly, we require to estimate the univariate's : $\phi_l$ for each $l \in \bivsuppvar$ such that $\degree(l) > 1$.
As above, for some strictly increasing sequences: $(-1 = \tp_1,\tp_2,\dots,\tp_{n_1} = 1)$, $(-1 = t_1,t_2,\dots,t_{n_1} = 1)$, 
consider the set 
\begin{equation} \label{eq:biv_univ_est_set}
\baseset_{l} := \Biggl\{\vecx_{i,j} \in \matR^{\dimn}: (\vecx_{i,j})_q =  \left\{
\begin{array}{rl}
\tp_i ; & q = l, \\
t_j ; & q \neq l \ \& \ q \in \bivsuppvar, \\
0; & q \notin \bivsuppvar, 
\end{array} \right\} ; \\ 1 \leq i,j \leq n_1; 1 \leq q \leq \dimn \Biggr\}; \quad l \in \bivsuppvar: \degree(l) > 1.
\end{equation}
We obtain $\set{f(\vecx_{i,j})}_{i,j=1}^{n_1}$; $\vecx_{i,j} \in \baseset_{l}$ where this time 
\begin{align}
 f(\vecx_{i,j}) &= \phi_l(\tp_i) + \sum_{\degree(\lp) > 1, \lp \neq l} \phi_{\lp}(t_j) + \sum_{\lp:\lpair \in \bivsupp} \phi_{\lpair}(\tp_i,t_j) \\
&+ \sum_{\lp:\lpairi \in \bivsupp} \phi_{\lpairi}(t_j,\tp_i) + \sum_{\qpair \in \bivsupp : q,\qp \neq l} \phi_{\qpair}(t_j,t_j) + C \\
&= g_l(\tp_i,t_j) + C
\end{align}
for a constant, $C$. Denoting $n = n_1^2$ to be the total number of samples of $f$, we can again obtain an estimate 
$\phitil_l(x_l,x)$ to $g_l(x_l,x) + C$, with $\Linfnorm[-1,1]^2$ error rate $O(n^{-3/2})$.
Then with $\phitil_l$ at hand, we define the estimate $\est{\phi}_l: [-1,1] \rightarrow \matR$ as
\begin{equation} \label{eq:biv_univ_est}
\est{\phi}_l := \expec_x[\phitil_l] - \expec_{(l,x)}[\phitil_l]; \quad l \in \bivsuppvar : \degree(l) > 1. 
\end{equation}
The following proposition formally describes the error rates for the aforementioned estimates.

%
\begin{proposition} \label{prop:no_nois_est_comp}
For $C^3$ smooth components $\phi_p, \phi_{\lpair},\phi_{l}$,  let $\est{\phi}_p$, $\est{\phi}_{\lpair}, \est{\phi}_{l}$ 
be the respective estimates as defined in \eqref{eq:univ_est}, \eqref{eq:biv_est} and \eqref{eq:biv_univ_est} respectively.
Also, let $n$ denote the number of queries (of $f$) made per component. We then have that: 
\begin{enumerate}
\item $\norm{\est{\phi}_p - \phi_p}_{\Linfnorm[-1,1]} = O(n^{-3}); \forall p \in \univsupp$, 
\item $\norm{\est{\phi}_{\lpair} - \phi_{\lpair}}_{\Linfnorm[-1,1]^2} = O(n^{-3/2}); \forall \lpair \in \bivsupp$, and 
\item $\norm{\est{\phi}_{l} - \phi_{l}}_{\Linfnorm[-1,1]} = O(n^{-3/2}); \forall l \in \bivsuppvar: \degree(l) > 1$.
\end{enumerate}
\end{proposition}
\begin{proof}

\begin{enumerate}
\item $\mathbf{p \in \univsupp}$.

We have for $\phitil_p$ that $\norm{\phitil_p - (\phi_p + C)}_{\Linfnorm[-1,1]} = O(n^{-3})$. 
Denoting $\phitil_p(x_p) - (\phi_p(x_p) + C) = z_p(x_p)$, this means $\abs{z_p(x_p)} = O(n^{-3})$, $\forall x_p \in [-1,1]$.
Now $\abs{\expec_p[\phitil_p - (\phi_p + C)]} = \abs{\expec_p[\phitil_p] - C} = \abs{\expec_p[z_p]} \leq \expec_p[\abs{z_p}] = O(n^{-3})$. 

Lastly, we have that:
\begin{align}
\norm{\est{\phi}_p - \phi_p}_{\Linfnorm[-1,1]} &= \norm{\phitil_p - \expec_p[\phitil_p] - \phi_p}_{\Linfnorm[-1,1]} \\ 
&= \norm{\phitil_p - (\phi_p + C) - (\expec_p[\phitil_p] - C)}_{\Linfnorm[-1,1]} \\ &= O(n^{-3}).
\end{align}

\item $\mathbf{\lpair \in \bivsupp}$.

We only consider the case where $\degree(l), \degree(\lp) > 1$ as proofs for the other cases are similar. 
Now for $\phitil_{\lpair}$ we have that $\norm{\phitil_{\lpair} - (g_{\lpair} + C)}_{\Linfnorm[-1,1]^2} = O(n^{-3/2})$. 
Denoting $\phitil_{\lpair}\xlpair - (g_{\lpair}\xlpair + C) = z_{\lpair} \xlpair$, this means 
$\abs{z_{\lpair}\xlpair} = O(n^{-3/2})$, $\forall \xlpair \in [-1,1]^2$. 
Consequently, one can easily verify that:

\begin{align}
 \norm{\expec_l[\phitil_{\lpair}] - (\expec_l[g_{\lpair}] + C)}_{\Linfnorm[-1,1]} = O(n^{-3/2}), \label{eq:temp5}\\
 \norm{\expec_{\lp}[\phitil_{\lpair}] - (\expec_{\lp}[g_{\lpair}] + C)}_{\Linfnorm[-1,1]} = O(n^{-3/2}), \label{eq:temp6}\\
 \norm{\expec_{\lpair}[\phitil_{\lpair}] - (\expec_{\lpair}[g_{\lpair}] + C)}_{\Linfnorm} = O(n^{-3/2}). \label{eq:temp7}
\end{align}

Now note that using the form for $g_{\lpair}$ from \eqref{eq:biv_part_fn_exp}, we have that 

\begin{align}
\expec_{l}[g_{\lpair}] &= \sum_{\substack{l_1:(l,l_1) \in \bivsupp \\ l_1 \neq \lp}} \expec_{l}[\phi_{(l,l_1)} (x_l,0)] + \sum_{\substack{l_1:(l_1,l) \in \bivsupp \\ l_1 \neq \lp}} 
\expec_l[\phi_{(l_1,l)} (0,x_l)] + \sum_{\substack{\lp_1:(\lp,\lp_1) \in \bivsupp \\ \lp_1 \neq l}} \phi_{(\lp,\lp_1)} (x_{\lp},0) \nonumber \\ 
&+ \sum_{\substack{\lp_1:(\lp_1,\lp) \in \bivsupp \\ \lp_1 \neq l}} \phi_{(\lp_1,\lp)} (0,x_{\lp}) + 
\phi_{\lp}(x_{\lp}) + C, \quad \text{and} \label{eq:temp1} \\
\expec_{\lp}[g_{\lpair}] &= \sum_{\substack{l_1:(l,l_1) \in \bivsupp \\ l_1 \neq \lp}} \phi_{(l,l_1)} (x_l,0) + \sum_{\substack{l_1:(l_1,l) \in \bivsupp \\ l_1 \neq \lp}} 
\phi_{(l_1,l)} (0,x_l) 
+ \sum_{\substack{\lp_1:(\lp,\lp_1) \in \bivsupp \\ \lp_1 \neq l}} \expec_{\lp} [\phi_{(\lp,\lp_1)} (x_{\lp},0)] \nonumber \\ 
&+ \sum_{\substack{\lp_1:(\lp_1,\lp) \in \bivsupp \\ \lp_1 \neq l}} \expec_{\lp}\phi_{(\lp_1,\lp)} (0,x_{\lp}) + 
\phi_l(x_l) + C, \quad \text{and} \label{eq:temp2} \\
\expec_{\lpair}[g_{\lpair}] &= \sum_{\substack{l_1:(l,l_1) \in \bivsupp \\ l_1 \neq \lp}} \expec_{l}[\phi_{(l,l_1)} (x_l,0)] + \sum_{\substack{l_1:(l_1,l) \in \bivsupp \\ l_1 \neq \lp}} 
\expec_l[\phi_{(l_1,l)} (0,x_l)] \nonumber \\ 
&+ \sum_{\substack{\lp_1:(\lp,\lp_1) \in \bivsupp \\ \lp_1 \neq l}} \expec_{\lp} [\phi_{(\lp,\lp_1)} (x_{\lp},0)]  
+ \sum_{\substack{\lp_1:(\lp_1,\lp) \in \bivsupp \\ \lp_1 \neq l}} \expec_{\lp}\phi_{(\lp_1,\lp)} (0,x_{\lp}) +  C. \label{eq:temp3}
\end{align}
We then have from \eqref{eq:biv_part_fn_exp}, \eqref{eq:temp1}, \eqref{eq:temp2}, \eqref{eq:temp3} that 
\begin{equation}
g_{\lpair} - \expec_{l}[g_{\lpair}] - \expec_{\lp}[g_{\lpair}] + \expec_{\lpair}[g_{\lpair}] = \phi_{\lpair}. \label{eq:temp4}
\end{equation}
Using \eqref{eq:temp5}, \eqref{eq:temp6}, \eqref{eq:temp7}, \eqref{eq:temp4}, and \eqref{eq:biv_est} it then follows that:
\begin{align}
\norm{\est{\phi}_{\lpair} - \phi_{\lpair}}_{\Linfnorm[-1,1]^2} = O(n^{-3/2}).
\end{align}

\item $\mathbf{l \in \bivsuppvar: \degree(l) > 1}$.

In this case, for $\phitil_l : [-1,1]^2 \rightarrow \matR$, we have that 
$\norm{\phitil_l -(g_l + C)}_{\Linfnorm[-1,1]^2} = O(n^{-3/2})$, with 
\begin{align}
g_{l}(x_l,x) =  \phi_l(x_l) &+ \sum_{\degree(\lp) > 1, \lp \neq l} \phi_{\lp}(x) + \sum_{\lp:\lpair \in \bivsupp} \phi_{\lpair}(x_l,x) \nonumber \\
&+ \sum_{\lp:\lpairi \in \bivsupp} \phi_{\lpairi}(x,x_l) + \sum_{\qpair \in \bivsupp : q,\qp \neq l} \phi_{\qpair}(x,x). \label{eq:temp11}
\end{align}
From \eqref{eq:temp11}, we see that: 
\begin{align}
\expec_{x}[g_l(x_l,x)] &= \phi_l(x_l) + \sum_{\qpair \in \bivsupp : q,\qp \neq l} \expec_x[\phi_{\qpair}(x,x)], \\
\text{and} \quad \expec_{(l,x)}[g_l(x_l,x)] &= \sum_{\qpair \in \bivsupp : q,\qp \neq l} \expec_{x}[\phi_{\qpair}(x,x)]. 
\end{align}
Hence clearly, $\expec_{x}[g_l(x_l,x)] - \expec_{(l,x)}[g_l(x_l,x)] = \phi_l(x_l)$. One can also easily verify that 
\begin{align}
\norm{\expec_{x}[\phitil_l] - (\expec_{x}[g_l] + C)}_{\Linfnorm[-1,1]} &= O(n^{-3/2}), \\ 
\norm{\expec_{(l,x)}[\phitil_l] - (\expec_{(l,x)}[g_l] + C)}_{\Linfnorm}  &= O(n^{-3/2}). 
\end{align}
Therefore it follows that 
\begin{align}
\norm{\est{\phi}_{l} - \phi_{l}}_{\Linfnorm[-1,1]} 
&= \norm{(\expec_{x}[\phitil_l] - \expec_{(l,x)}[\phitil_l]) - (\expec_{x}[g_l] - \expec_{(l,x)}[g_l])}_{\Linfnorm[-1,1]} \\
&\leq \norm{\expec_{x}[\phitil_l] - (\expec_{x}[g_l] + C)}_{\Linfnorm[-1,1]} + \norm{\expec_{(l,x)}[\phitil_l] - (\expec_{(l,x)}[g_l] + C)}_{\Linfnorm} \\
&= O(n^{-3/2}).
\end{align}
This completes the proof.
\end{enumerate}
\end{proof}

\subsection{Noisy queries}
We now look at the case where for each query $\vecx \in \matR^d$, we obtain a noisy value $f(\vecx) + \exnoisep$.

\paragraph{Arbitrary bounded noise.} We begin with the scenario where $\exnoisep_i$ is arbitrary and bounded 
with $\abs{\exnoisep_i} < \exnoisemag; \ \forall i$. Since the noise is arbitrary in nature, therefore we simply proceed 
\emph{as in the noiseless case}, i.e., by approximating each component via a quasi-interpolant. As the magnitude of the noise is bounded 
by $\exnoisemag$, it results in an additional $O(\exnoisemag)$ term in the approximation error rates of Proposition \ref{prop:no_nois_est_comp}. 

To see this for the univariate case, let us denote $Q: C(\matR) \rightarrow \spacespl$ to be a quasi-interpolant operator. 
This a linear operator, with $C(\matR)$ denoting the space of continuous functions defined over $\matR$ and $\spacespl$ denoting a univariate 
spline space. Consider $u \in C^m[-1,1]$ for some positive integer $m$, and let $g:[-1,1] \rightarrow \matR$ be an arbitrary continuous function with 
$\norm{g}_{\Linfnorm[-1,1]} < \exnoisemag$. Denote $\est{u} = u + g$ to be the ``corrupted'' version of $u$, 
and let $n$ be the number of samples of $\est{u}$ used by $Q$. We then have by linearity of $Q$ that: 

\begin{equation}
\norm{Q(\est{u}) - u}_{\Linfnorm[-1,1]} = \norm{Q(u) + Q(g) - u}_{\Linfnorm[-1,1]} \leq 
\underbrace{\norm{Q(u) - u}_{\Linfnorm[-1,1]}}_{= O(n^{-m})} + \norm{Q}\underbrace{\norm{g}_{\Linfnorm[-1,1]}}_{\leq \norm{Q}\exnoisemag}, 
\end{equation}

with $\norm{Q}$ being the operator norm of $Q$. One can construct $Q$ with $\norm{Q}$ 
bounded\savefootnote{foot:quasinterp_ref}{For instance, see Theorems $14.4, 15.2$ in \cite{Gyorfi2002}}
from above by a constant depending only on $m$. The above argument can be extended easily to the multivariate case. We state this for the bivariate 
case for completeness. Denote $Q_1: C(\matR^2) \rightarrow \spacespl$ to be a quasi-interpolant operator, with $\spacespl$ denoting a bivariate spline space. 
Consider $u_1 \in C^m[-1,1]^2$ for some positive integer $m$, and let $g_1:[-1,1] \rightarrow \matR$ be an arbitrary continuous function with 
$\norm{g_1}_{\Linfnorm[-1,1]^2} < \exnoisemag$. Let $\est{u}_1 = u_1 + g_1$ and let $n$ be the number of samples of $\est{u_1}$ used by $Q_1$. 
We then have by linearity of $Q_1$ that: 

\begin{equation}
\norm{Q_1(\est{u_1}) - u_1}_{\Linfnorm[-1,1]^2} = \norm{Q_1(u_1) + Q_1(g_1) - u_1}_{\Linfnorm[-1,1]^2} \leq 
\underbrace{\norm{Q_1(u_1) - u_1}_{\Linfnorm[-1,1]^2}}_{= O(n^{-m/2})} + \norm{Q_1}\underbrace{\norm{g_1}_{\Linfnorm[-1,1]^2}}_{\leq \norm{Q_1}\exnoisemag}, 
\end{equation}

with $\norm{Q_1}$ being the operator norm of $Q_1$. As for the univariate case, one can construct $Q_1$ with $\norm{Q_1}$ 
bounded\repeatfootnote{foot:quasinterp_ref} from above by a constant depending only on $m$. 

Let us define our final estimates $\est{\phi}_p$, $\est{\phi}_{\lpair}$ and $\est{\phi}_{l}$ as in 
\eqref{eq:univ_est}, \eqref{eq:biv_est} and \eqref{eq:biv_univ_est}, respectively. 
The following proposition formally states the error bounds, for this particular noise model.

\begin{proposition}[Arbitrary bounded noise] \label{prop:arb_nois_est_comp}
For $C^3$ smooth components $\phi_p, \phi_{\lpair},\phi_{l}$,  let $\est{\phi}_p$, $\est{\phi}_{\lpair}, \est{\phi}_{l}$ 
be the respective estimates as defined in \eqref{eq:univ_est}, \eqref{eq:biv_est} and \eqref{eq:biv_univ_est} respectively.
Also, let $n$ denote the number of noisy queries (of $f$) made per component with the external noise magnitude being bounded 
by $\exnoisemag$. We then have that 

\begin{enumerate}
\item $\norm{\est{\phi}_p - \phi_p}_{\Linfnorm[-1,1]} = O(n^{-3}) + O(\exnoisemag); \forall p \in \univsupp$, 
\item $\norm{\est{\phi}_{\lpair} - \phi_{\lpair}}_{\Linfnorm[-1,1]^2} = O(n^{-3/2}) + O(\exnoisemag); \forall \lpair \in \bivsupp$, and 
\item $\norm{\est{\phi}_{l} - \phi_{l}}_{\Linfnorm[-1,1]} = O(n^{-3/2}) + O(\exnoisemag); \forall l \in \bivsuppvar: \degree(l) > 1$. 
\end{enumerate}
\end{proposition}

The proof is similar to that of Proposition \ref{prop:no_nois_est_comp} and hence skipped.

\paragraph{Stochastic noise.} We now consider the setting where $\exnoisep_i \sim \calN(0,\sigma^2)$ 
are i.i.d Gaussian random variables. Similar to the noiseless case, 
estimating the individual components again involves sampling $f$ along the 
subspaces corresponding to $\univsupp$, $\bivsupp$. Due to the presence of stochastic noise however, 
we now make use of \emph{nonparametric regression} techniques to compute the estimates. 
While there exist a number of methods that could be used for this 
purpose (cf. \cite{tsyba08}), we only discuss a specific one for clarity of exposition.

To elaborate, we again construct the sets defined in \eqref{eq:univ_est_set},\eqref{eq:biv_est_set} and\eqref{eq:biv_univ_est_set}. 
In particular, we uniformly discretize the domains 
$[-1,1]$ and $[-1,1]^2$, by choosing the respective $t_i$'s and $(\tp_i,t_j)$'s accordingly. 
This is the so called ``fixed design'' setting in nonparametric statistics. 
Upon collecting the samples $\set{f(\vecx_i) + \exnoisep_i}_{i=1}^{n}$ one can then derive estimates $\phitil_p$, $\phitil_{\lpair}, \phitil_l$, to 
$\phi_p + C$, $g_{\lpair} + C$ and $g_l + C$ respectively, by using \emph{local polynomial 
estimators} (cf. \cite{tsyba08, Fan96} and references within). 
It is known that these estimators achieve the (minimax optimal) $\Linfnorm$ error rate: 
$\Omega((n^{-1} \log n)^{\frac{m}{2m+\dimn}})$, for estimating $d$-variate, $C^m$ smooth functions over compact 
domains\footnote{See \cite{tsyba08} for $\dimn=1$, and \cite{Nemi00} for $\dimn \geq 1$}. 
Translated to our setting, we then have that the functions: $\phi_p + C$, $g_{\lpair} + C$ and $g_l + C$ are estimated at the rates: 
$O((n^{-1} \log n)^{\frac{3}{7}})$ and $O((n^{-1} \log n)^{\frac{3}{8}})$ respectively.

Denoting the above intermediate estimates by $\phitil_{p}$, $\phitil_{\lpair}$, $\phitil_l$, we define our final estimates 
$\est{\phi}_p$, $\est{\phi}_{\lpair}$ and $\est{\phi}_{l}$ as in 
\eqref{eq:univ_est}, \eqref{eq:biv_est} and \eqref{eq:biv_univ_est}, respectively. 
The following Proposition describes the error rates of these estimates. 

\begin{proposition}[i.i.d Gaussian noise] \label{prop:gauss_nois_est_comp}
For $C^3$ smooth components $\phi_p, \phi_{\lpair},\phi_{l}$,  let $\est{\phi}_p$, $\est{\phi}_{\lpair}, \est{\phi}_{l}$ 
be the respective estimates as defined in \eqref{eq:univ_est}, \eqref{eq:biv_est} and \eqref{eq:biv_univ_est} respectively.
Let $n$ denote the number of noisy queries (of $f$) made per component, with noise samples $\exnoisep_1,\exnoisep_2,\dots,\exnoisep_n$ 
being i.i.d Gaussian. Furthermore, let $\expec_{z}[\cdot]$ denote expectation w.r.t the 
joint distribution of $\exnoisep_1,\exnoisep_2,\dots,\exnoisep_n$. We then have that 

\begin{enumerate}
\item $\expec_{z}[\norm{\est{\phi}_p - \phi_p}_{\Linfnorm[-1,1]}] = O((n^{-1} \log n)^{\frac{3}{7}}); \forall p \in \univsupp$, 
\item $\expec_{z}[\norm{\est{\phi}_{\lpair} - \phi_{\lpair}}_{\Linfnorm[-1,1]^2}] = O((n^{-1} \log n)^{\frac{3}{8}}); \forall \lpair \in \bivsupp$, and 
\item $\expec_{z}[\norm{\est{\phi}_{l} - \phi_{l}}_{\Linfnorm[-1,1]}] = O((n^{-1} \log n)^{\frac{3}{8}}); \forall l \in \bivsuppvar: \degree(l) > 1$.
\end{enumerate}
\end{proposition}

Although the proof is again very similar to that of Proposition \ref{prop:no_nois_est_comp}, there 
are some technical differences. Hence we provide a brief sketch of the proof, avoiding details already 
highlighted in the proof of Proposition \ref{prop:no_nois_est_comp}.

\begin{proof}
\begin{enumerate}
\item $\mathbf{p \in \univsupp}$.

We have for $\phitil_p$ that $\expec_{z}[\norm{\phitil_p - (\phi_p + C)}_{\Linfnorm[-1,1]}] = O((n^{-1} \log n)^{\frac{3}{7}})$. 
Denoting $\phitil_p(x_p) - (\phi_p(x_p) + C) = b_p(x_p)$, this means $\expec_{z}[\abs{b_p(x_p)}] = O((n^{-1} \log n)^{\frac{3}{7}})$. 
Now, 

\begin{equation}
\expec_{z}[\abs{\expec_p[\phitil_p - (\phi_p + C)]}] = \expec_{z}[\abs{\expec_p[b_p]}] \leq \expec_{z}[\expec_p[\abs{b_p}]] = 
\expec_{p}[\expec_z[\abs{b_p(x_p)}]] = O((n^{-1} \log n)^{\frac{3}{7}}). 
\end{equation}

The penultimate equality above involves swapping the order of expectations, which is possible by Tonelli's 
theorem (since $\abs{b_p} > 0$). Then using triangle inequality, it follows that 
$\expec_z[\norm{\est{\phi}_p - \phi_p}_{\Linfnorm[-1,1]}] = O((n^{-1} \log n)^{\frac{3}{7}})$.

\item $\mathbf{\lpair \in \bivsupp}$.

We only consider the case where $\degree(l), \degree(\lp) > 1$ as proofs for the cases are similar. 
For $\phitil_{\lpair}$, we have that 
$\expec_z[\norm{\phitil_{\lpair} - (g_{\lpair} + C)}_{\Linfnorm[-1,1]^2}] = O((n^{-1} \log n)^{\frac{3}{8}})$. 
Denoting $\phitil_{\lpair}\xlpair - (g_{\lpair}\xlpair + C) = b_{\lpair} \xlpair$, this means 
$\expec_z[\abs{b_{\lpair}\xlpair}] = O((n^{-1} \log n)^{\frac{3}{8}})$, $\forall \xlpair \in [-1,1]^2$. 
Using Tonelli's theorem as earlier, one can next verify that:
\begin{align}
 \expec_z[\norm{\expec_l[\phitil_{\lpair}] - (\expec_l[g_{\lpair}] + C)}_{\Linfnorm[-1,1]}] = O((n^{-1} \log n)^{\frac{3}{8}}), \label{eq:gauss_temp5} \\
 \expec_z[\norm{\expec_{\lp}[\phitil_{\lpair}] - (\expec_{\lp}[g_{\lpair}] + C)}_{\Linfnorm[-1,1]}] = O((n^{-1} \log n)^{\frac{3}{8}}), \label{eq:gauss_temp6} \\
 \expec_z[\abs{\expec_{\lpair}[\phitil_{\lpair}] - (\expec_{\lpair}[g_{\lpair}] + C)}] = O((n^{-1} \log n)^{\frac{3}{8}}). \label{eq:gauss_temp7}
\end{align}
As in the proof of Proposition \ref{prop:no_nois_est_comp}, we obtain from \eqref{eq:gauss_temp5}, \eqref{eq:gauss_temp6}, 
\eqref{eq:gauss_temp7}, \eqref{eq:temp4}, \eqref{eq:biv_est} (via triangle inequality):
\begin{align}
\expec_z[\norm{\est{\phi}_{\lpair} - \phi_{\lpair}}_{\Linfnorm[-1,1]^2}] = O((n^{-1} \log n)^{\frac{3}{8}}).
\end{align}

\item $\mathbf{l \in \bivsuppvar: \degree(l) > 1}$.

In this case, for $\phitil_l : [-1,1]^2 \rightarrow \matR$, we have that 
$\expec_z[\norm{\phitil_l - (g_l + C)}_{\Linfnorm[-1,1]^2}] = O((n^{-1} \log n)^{\frac{3}{8}})$, with 
$g_l(x_l,x)$ as defined in \eqref{eq:temp11}. Using Tonelli's theorem as earlier, one can verify that  
\begin{align}
\expec_z[\norm{\expec_{x}[\phitil_l] - (\expec_{x}[g_l] + C)}_{\Linfnorm[-1,1]}] &= O((n^{-1} \log n)^{\frac{3}{8}}), \\ 
\expec_z[\abs{\expec_{(l,x)}[\phitil_l] - (\expec_{(l,x)}[g_l] + C)}] &= O((n^{-1} \log n)^{\frac{3}{8}}). 
\end{align}
Then using the fact $\expec_{x}[g_l(x_l,x)] - \expec_{(l,x)}[g_l(x_l,x)] = \phi_l(x_l)$, 
we obtain via triangle inequality the bound: 
$\expec_z[\norm{\est{\phi}_{l} - \phi_{l}}_{\Linfnorm[-1,1]}] = O((n^{-1} \log n)^{\frac{3}{8}})$. This completes the proof.
\end{enumerate}
\end{proof} 

\end{document}